\documentclass{article}
\usepackage{times}
\usepackage{subcaption}
\usepackage{caption}

\usepackage[utf8]{inputenc} 
\usepackage[T1]{fontenc}    
\usepackage{hyperref}       
\usepackage{url}            
\usepackage{booktabs}       
\usepackage{amsfonts,amsmath,amssymb}       
\usepackage{nicefrac}       
\usepackage{microtype}      
\usepackage{xcolor}         
\usepackage{txfonts}
\usepackage{graphicx}
\usepackage{wrapfig,bm,comment,color}
\usepackage{breakurl,epsfig,epsf,fmtcount,semtrans,multirow,boldline}
\usepackage{tcolorbox}
\tcbuselibrary{skins}
\usepackage{tikz}
\usepackage{titletoc}

\usepackage[normalem]{ulem}

\usepackage[textsize=tiny]{todonotes}

\definecolor{darkred}{RGB}{150,0,0}
\definecolor{darkgreen}{RGB}{0,150,0}
\definecolor{darkblue}{RGB}{0,0,200}
\hypersetup{colorlinks=true, linkcolor=darkred, citecolor=darkgreen, urlcolor=darkblue}

\newtheorem{theorem}{Theorem}

\newtheorem{fact}{Fact}

\newtheorem{assumption}{Assumption}

\newtheorem{lemma}{Lemma}
\newtheorem{corollary}{Corollary}
\newtheorem{proposition}{Proposition}

\newcommand{\beq}{\begin{equation}}
\newcommand{\ba}{\begin{align}}
\newcommand{\ea}{\end{align}}

\newcommand{\eeq}{\end{equation}}

\def \endprf{\hfill {\vrule height6pt width6pt depth0pt}\medskip}

\newenvironment{proof}{\noindent {\bf Proof.} }{\endprf\par}
\newenvironment{proofof}[1]{\noindent {\bf Proof of {#1}.} }{\endprf\par}

\newcommand{\vct}[1]{\bm{#1}}
\newcommand{\tr}[1]{\texttt{tr}\left(#1\right)}
\newcommand{\mtx}[1]{\bm{#1}}
\newcommand{\red}{\textcolor{darkred}}

\newcommand{\redp}[1]{\textcolor{darkred}{[#1]}}

\newcommand{\eps}{\varepsilon}

\newcommand{\beps}{\boldsymbol{\eps}}
\newcommand{\bxi}{\boldsymbol{\xi}}

\newcommand{\hb}{\vct{h}}

\newcommand{\st}{\star}

\newcommand{\la}{\lambda}
\newcommand{\A}{{\mtx{A}}}

\newcommand{\B}{{\mtx{B}}}
\newcommand{\Bb}{{\mtx{\bar{B}}}}

\newcommand{\Ub}{{\mtx{U}}}

\newcommand{\Sb}{{{\mtx{S}}}}
\newcommand{\Sbb}{{{\bar{\mtx{S}}}}}

\newcommand{\Lc}{{\cal{L}}}
\newcommand{\Lcb}{\bar{\cal{L}}}
\newcommand{\Lcs}{{\cal{L}_\st}}

\newcommand{\Nc}{{\cal{N}}}

\newcommand{\Cb}{{\mtx{C}}}
\newcommand{\Eb}{{\mtx{E}}}

\newcommand{\La}{{\boldsymbol{\Lambda}}}

\newcommand{\Db}{{\mtx{D}}}

\newcommand{\onebb}{{\mathbf{1}}}
\newcommand{\zerbb}{{\mathbf{0}}}
\newcommand{\Iden}{{\mtx{I}}}

\newcommand{\z}{{\vct{z}}}

\newcommand{\bt}{{\boldsymbol{\beta}}}

\newcommand{\btb}{\bar{\boldsymbol{\beta}}}

\newcommand{\bPi}{{\boldsymbol{\Pi}}}

\newcommand{\Nn}{\mathcal{N}}
\newcommand{\vb}{\vct{v}}

\newcommand{\fb}{\vct{f}}

\newcommand{\Ic}{{\mathcal{I}}}

\newcommand{\w}{\vct{w}}
\newcommand{\Wlr}{\vct{W}_{\text{lora}}}
\newcommand{\Wbr}{\vct{\bar{W}}_{\text{lora}}}
\newcommand{\wb}{\bar{\vct{w}}}

\newcommand{\s}{\vct{s}}
\newcommand{\ab}{\vct{a}}
\newcommand{\Lb}{\vct{L}}

\newcommand{\ub}{{\vct{u}}}

\newcommand{\g}{{\vct{g}}}

\newcommand{\Z}{\mtx{Z}}
\newcommand{\Zm}{\mtx{Z}_0}

\newcommand{\bSi}{\boldsymbol{\Sigma}}
\newcommand{\bSb}{\boldsymbol{\bar{\Sigma}}}
\newcommand{\bSp}{\boldsymbol{\tilde{\Sigma}}}

\newcommand{\x}{\vct{x}}

\newcommand{\y}{\vct{y}}

\newcommand{\W}{\mtx{W}}
\newcommand{\WK}{\mtx{W}_k}
\newcommand{\WQ}{\mtx{W}_q}
\newcommand{\WV}{\mtx{W}_v}
\newcommand{\WLR}{\mtx{W}_{lora}}

\newcommand{\WKh}{\hat{\mtx{W}}_k}
\newcommand{\WQh}{\hat{\mtx{W}}_q}
\newcommand{\WVh}{\hat{\mtx{W}}_v}
\newcommand{\Wz}{\mtx{W}_0}
\newcommand{\Wc}{{\cal{W}}}

\newcommand{\X}{{\mtx{X}}}

\newcommand{\Vb}{{\mtx{V}}}

\newcommand{\R}{\mathbb{R}}

\newcommand{\E}{\operatorname{\mathbb{E}}}

\newcommand{\nn}{\nonumber}

\newcommand{\diag}[1]{\text{diag}\left(#1\right)}

\newcommand{\tn}[1]{\left\|{#1}\right\|_{\ell_2}}
\newcommand{\tf}[1]{\left\|{#1}\right\|_{F}}

\newcommand{\order}[1]{{\cal{O}}\left(#1\right)}

\newcommand{\bgl}{{~\big |~}}

\newcommand{\Ws}{\W_\star}

\newcommand{\lab}{\bar{\la}}
\newcommand{\Wt}{\tilde{\mtx{W}}}

\newcommand{\Wb}{\mtx{\bar{W}}}

\newcommand{\Xb}{\mtx{\bar{X}}}

\newcommand{\xb}{\vct{\bar{x}}}

\newcommand{\bom}{\boldsymbol{\omega}}

\newcommand{\gd}{{\texttt{PGD}}}
\newcommand{\wpgd}{{\texttt{WPGD}}}
\newcommand{\att}{{\texttt{ATT}}}
\newcommand{\ssm}{{\texttt{SSM}}}

\newcommand{\trace}[1]{\texttt{tr}\left(#1\right)}
\newcommand{\rank}[1]{\text{rank}\left(#1\right)}
\newcommand\scalemath[2]{\scalebox{#1}{\mbox{\ensuremath{\displaystyle #2}}}}

\newcommand{\todoasr}[1]{} 
\newcommand{\ankit}[1]{}
\newcommand{\asrnote}[1]{}
\newcommand{\yl}[1]{}

\usepackage[preprint]{neurips_2024}
\usepackage[toc,page,header]{appendix}
\usepackage{minitoc}

\usepackage[utf8]{inputenc} 
\usepackage[T1]{fontenc}    
\usepackage{hyperref}       
\usepackage{url}            
\usepackage{booktabs}       
\usepackage{amsfonts}       
\usepackage{nicefrac}       
\usepackage{microtype}      
\usepackage{xcolor}         
\usepackage{enumitem}

\title{Fine-grained Analysis of In-context Linear Estimation:\\Data, Architecture, and Beyond}

\author{Yingcong Li\\University of Michigan\\\texttt{yingcong@umich.edu}\And Ankit Singh Rawat\\Google Research NYC\\\texttt{ankitsrawat@google.com}\And Samet Oymak\\University of Michigan\\\texttt{oymak@umich.edu}}

\begin{document}

\doparttoc 
\faketableofcontents 

\maketitle

\begin{abstract}
Recent research has shown that Transformers with linear attention are capable of in-context learning (ICL) by implementing a linear estimator through gradient descent steps. However, the existing results on the optimization landscape apply under stylized settings where task and feature vectors are assumed to be IID and the attention weights are fully parameterized. In this work, we develop a stronger characterization of the optimization and generalization landscape of ICL through contributions on architectures, low-rank parameterization, and correlated designs: (1) We study the landscape of 1-layer linear attention and 1-layer H3, a state-space model. Under a suitable correlated design assumption, we prove that both implement 1-step preconditioned gradient descent. We show that thanks to its native convolution filters, H3 also has the advantage of implementing sample weighting and outperforming linear attention in suitable settings. (2) By studying correlated designs, we provide new risk bounds for retrieval augmented generation (RAG) and task-feature alignment which reveal how ICL sample complexity benefits from distributional alignment. (3) We derive the optimal risk for low-rank parameterized attention weights in terms of covariance spectrum. Through this, we also shed light on how LoRA can adapt to a new distribution by capturing the shift between task covariances. Experimental results corroborate our theoretical findings. Overall, this work explores the optimization and risk landscape of ICL in practically meaningful settings and contributes to a more thorough understanding of its mechanics.
\end{abstract}

\begin{figure}[htbp]
\vspace{-9pt}
\includegraphics[width=1\linewidth,trim={-0.9cm 0 0 0},clip]{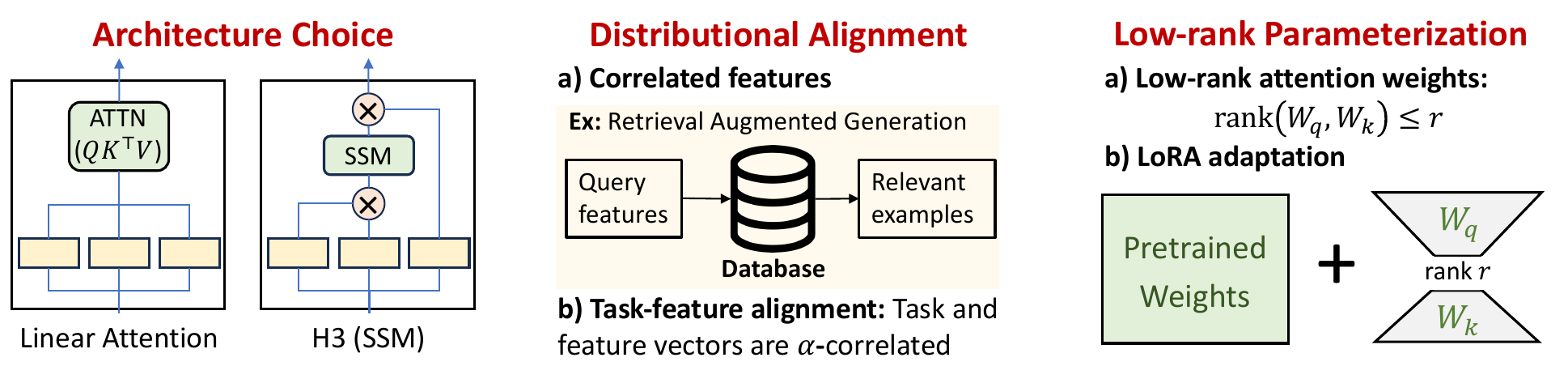}
\begin{subfigure}{0.325\linewidth}    
\centering
\includegraphics[height=0.7\linewidth]{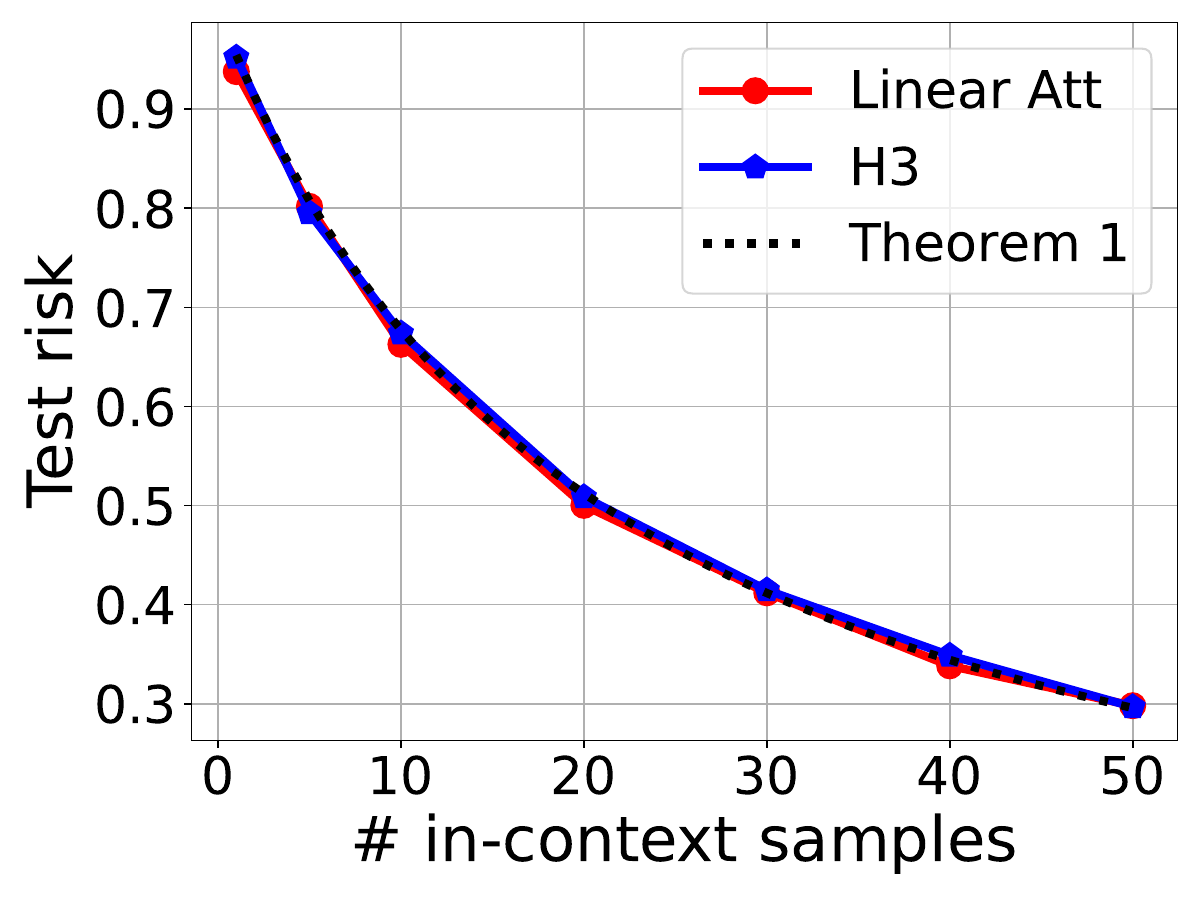}
\vspace{-7pt}
\caption{Linear attention$=$H3}\label{intro fig:iid}
\end{subfigure}
\begin{subfigure}{0.325\linewidth}    
    \centering
    \includegraphics[height=.7\columnwidth]{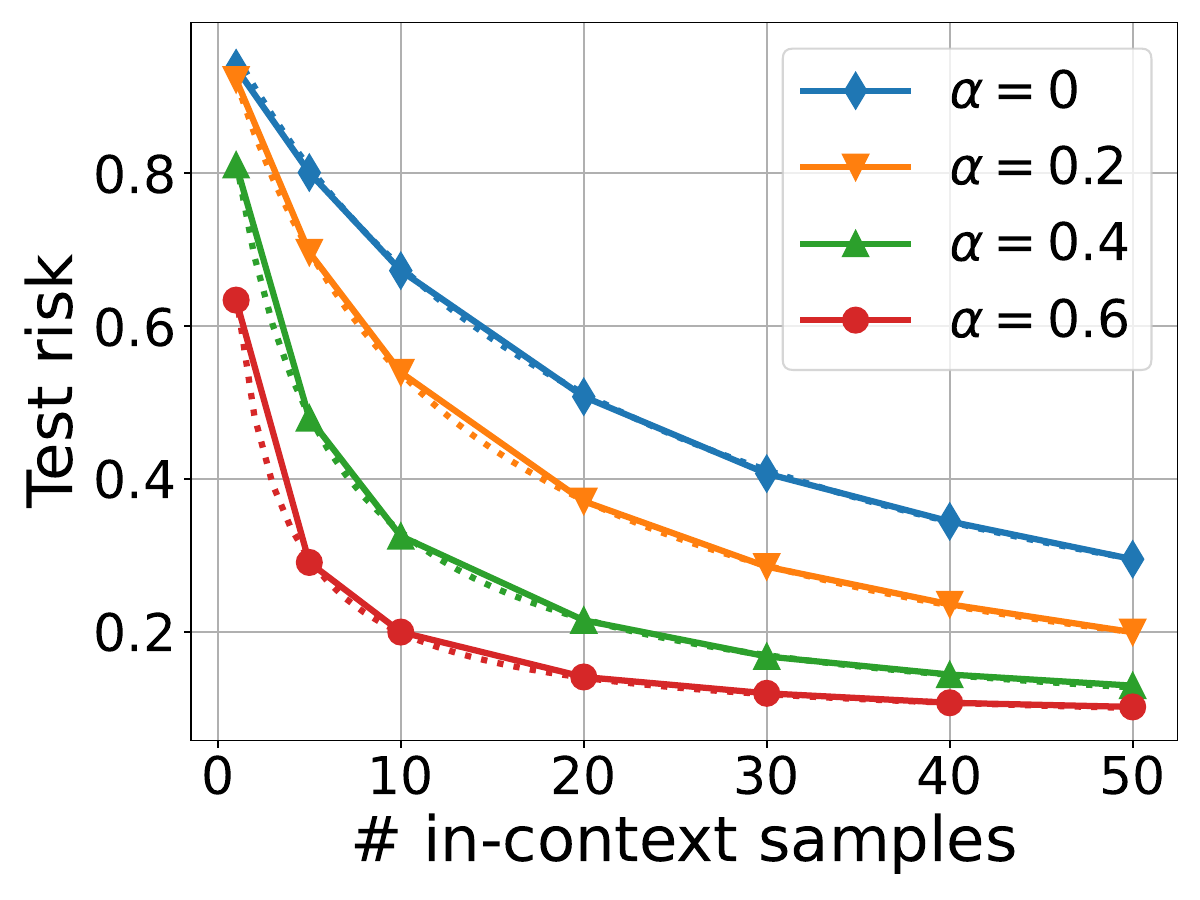}
\vspace{-7pt}
\caption{RAG with $\alpha$ correlation}\label{intro fig:rag loss}
\end{subfigure}
\begin{subfigure}{0.325\linewidth}    
\centering
\includegraphics[height=.7\columnwidth]{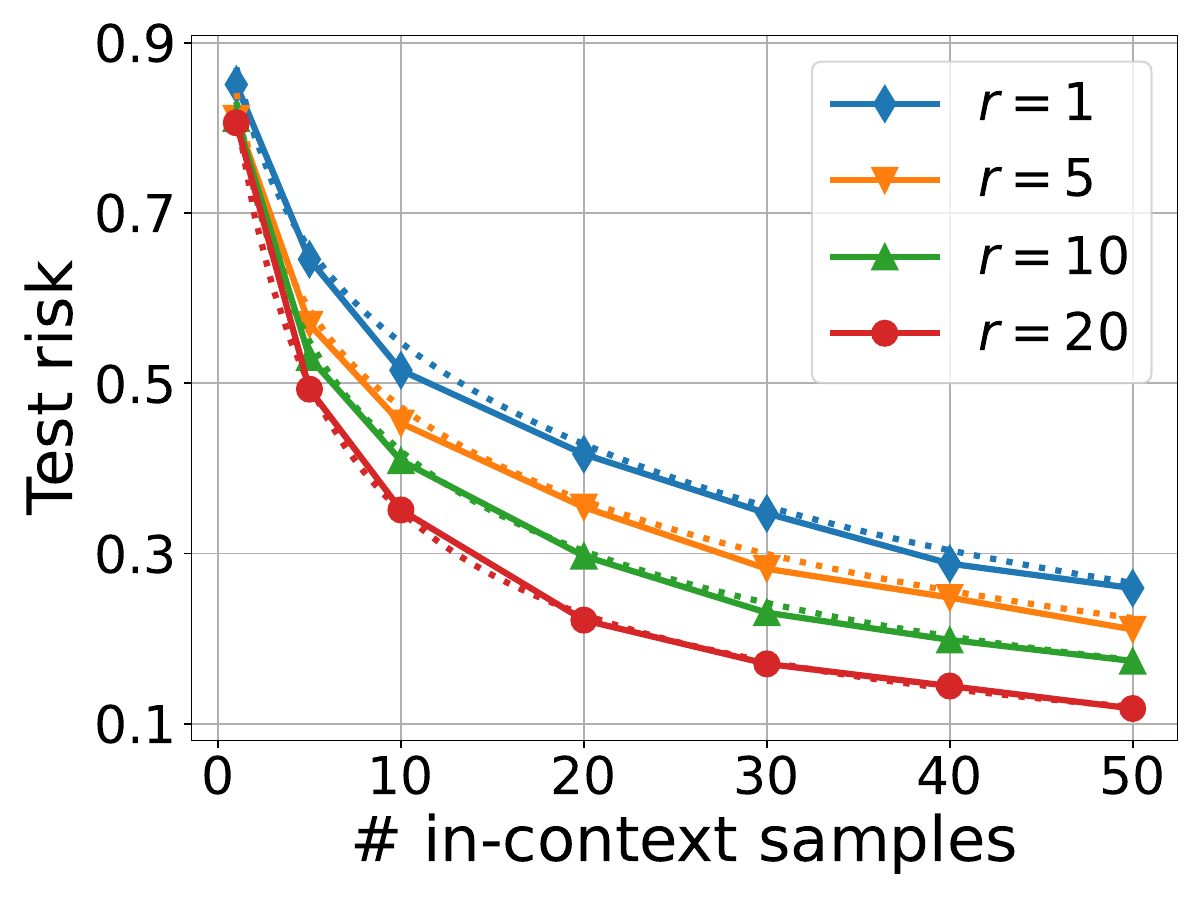}
\vspace{-7pt}
\caption{LoRA}\label{intro fig:lora}
\end{subfigure}
\caption{We investigate the optimization landscape of in-context learning from the lens of architecture choice, the role of distributional alignment, and low-rank parameterization. The empirical performance (solid curves) are aligned with our theoretical results (dotted curves) from Section~\ref{sec:main}. More experimental details and discussion are deferred to Section~\ref{sec exp}.}\label{fig:main}
\vspace{-25pt}
\end{figure}

\section{Introduction}

Modern language models exhibit the remarkable ability to learn novel tasks or solve complex problems from the demonstrations provided within their context window~\citep{brown2020language,team2023gemini,gpt4_techreport,touvron2023_llama}. Such \textit{in-context learning} (ICL) offers a novel and effective alternative to traditional fine-tuning techniques. It enables successful prediction across a wide range of tasks simply through a forward pass, eliminating the need for task-specific model weight updates.
Since its introduction, ICL capability has become an important feature of LLM with its applications spanning
retrieval-augmented generation~\citep{lewis2020retrieval}, and reasoning via advanced prompting techniques, such as chain-of-thought~\citep{wei2022chain}. While ICL already exhibits considerable benefits with a small number of demonstrations, i.e., few-shot data, there is a growing interest in extending its benefits to the many-shot settings, potentially realizing even more pronounced benefits~\citep{agarwal2024many}.

ICL ability also presents an important research avenue to develop stronger theoretical and mechanistic understanding of large language models. To this aim, there has been significant recent interest in demystifying ICL through the lens of function approximation \citep{liu2023transformers},\asrnote{@Samet: Not sure if you had any specific reference in mind for this?} Bayesian inference \citep{muller2021transformers,xie2022an, han2023context}, and learning and optimization theory \citep{ahn2024transformers, mahankali2024one, zhang2024trained, duraisamy2024finite}. The latter is concerned with understanding the optimization landscape of ICL, which is also crucial for understanding the generalization properties of the model. 
A notable result in this direction is the observation that linear attention models~\citep{schlag2021linear,von2023transformers,ahn2024transformers} implement \textit{preconditioned gradient descent} (PGD) during ICL \citep{ahn2024transformers,mahdavi2024revisiting}. While this line of works provide a fresh perspective to ICL, the existing studies do not address many questions arising from real-life applications nor provide guiding principles for various ICL setups motivated by practical considerations.


To this aim, we revisit the theoretical exploration of ICL with linear data model 
where we feed an in-context prompt containing $n$ examples $(\x_i,y_i = \x_i^\top\bt+\xi_i)_{i=1}^{n}\subset\R^{d}\times \R$ and a test instance or query $\x_{n+1} \in \R^d$ to the model, with $d$ being the feature dimension, $\bt\in\R^d$ being the task weight vector, and $(\xi_i)_{i=1}^n$ denoting the noise in individual labels. Given the in-context prompt, the model is tasked to predict $\hat{y}_{n+1}$ -- an estimate for $y_{n+1}=\x_{n+1}^\top\bt+\xi_{n+1}$.
We aim to provide answers to the following questions by exploring the loss landscape of ICL:
\begin{enumerate}[label=(Q\arabic*),leftmargin=27pt]
    \item Is the ability to implement gradient-based ICL unique to (linear) attention? Can alternative sequence models implement richer algorithms beyond PGD?
    \item In language modeling, ICL often works well with few-shot samples whereas standard linear estimation typically requires $\order{d}$ samples. How can we reconcile this discrepancy between classical learning and ICL? 
    \item To our knowledge, existing works assume linear-attention is fully parameterized, i.e.,~key and query projections $\W_k,\W_q\in\R^{d\times d}$. What happens when they are low-rank? What happens when there is distribution shift {between training and test in-context prompts} and we use LoRA~\citep{hu2022lora} for adaptation?
\end{enumerate}

In this work, we conduct a careful investigation of these questions. Specifically, we focus on ICL with 1-layer models and make the following contributions:
\begin{enumerate}[label=(A\arabic*),leftmargin=27pt]
\item We jointly investigate the landscape of linear attention and H3~\citep{fu2023hungry}, a widely popular state-space model (SSM). We prove that under correlated design, both models implement 1-step PGD (c.f. Proposition~\ref{lemma:eqv}) and the alignments in Fig.~\ref{intro fig:iid} verify that where the dotted curve represents the theoretical PGD result derived from Theorem~\ref{thm:independent}. Our analysis reveals that the gating mechanism in H3 imitates attention. We also empirically show that H3 has the advantage of implementing sample-weighting which allows it to outperform linear attention in temporally-heterogeneous problem settings in Section~\ref{sec exp} and Figure~\ref{fig:add exp}.
\item Proposition~\ref{lemma:eqv} allows for task and features to be correlated to each other as long as odd moments are zero. Through this, we can assess the impact of distributional alignment on the sample complexity of ICL. {Specifically, we characterize the performance of \emph{Retrieval Augmented Generation} (RAG) (c.f.~Theorem~\ref{thm:rag} and Fig.~\ref{intro fig:rag loss}) and \emph{Task-Feature Alignment} (c.f.~Theorem~\ref{thm:feature}), where the in-context examples are $\alpha$-correlated with either the query or the task vector. For both settings, we prove that alignment amplifies the \emph{effective sample size} of ICL by a factor of $\alpha^2d+1$, highlighting that aligned data are crucial for the success of ICL in few-shot settings.}
\item We show that, under low-rank parameterization, optimal attention-weights still implements PGD according to the truncated eigenspectrum of the fused task-feature covariance (see Section~\ref{sec:low rank}). We similarly derive risk upper bounds for LoRA adaptation (c.f.~Eq.~\eqref{formula lora} and Fig.~\ref{intro fig:lora}), and show that, these bounds accurately predict the empirical performance.
\end{enumerate}

\section{Problem Setup and Preliminaries}\label{sec:setup}
We begin with a short note on notation. 
Let bold lowercase and uppercase letters (e.g., $\x$ and $\X$) represent vectors and matrices, respectively. The symbol $\odot$ is defined as the element-wise (Hadamard) product, and $\ast$ denotes the convolution operator. $\onebb_d$ and $\zerbb_d$ denote the $d$-dimensional all-ones and all-zeros vectors, respectively; and $\Iden_d$ denotes the identity matrix of dimension $d\times d$. Additionally, let $\tr{\W}$ denote the trace of the square matrix $\W$.

As mentioned earlier, we study the optimization landscapes of 1-layer linear attention~\citep{katharopoulos2020transformers,schlag2021linear} and H3~\citep{fu2023hungry} models when training with prompts containing in-context data following a linear model.  We construct the input in-context prompt similar to~~\citet{ahn2024transformers,mahankali2024one,zhang2024trained} as follows.

\noindent\textbf{Linear data distribution.} Let $(\x,y)\in\R^{d}\times\R$ be a (feature, label) pair generated by a $d$-dimensional linear model parameterized by $\bt\in\R^{d}$, i.e., $y=\x^\top\bt+\xi$, where $\x$ and $\bt$ are feature and task vectors, and $\xi$ is the label noise.
Given demonstrations $(\x_i,y_i)_{i=1}^{n+1}$ sampled from a single $\bt$, define the input in-context prompt 
\begin{align}
\Z=[\z_1~\dots~\z_n~\z_{n+1}]^\top=\begin{bmatrix}
    \x_1&\dots&\x_n&\x_{n+1}\\
    y_1&\dots&y_n&0
\end{bmatrix}^\top \in \R^{(n+1) \times (d+1)}.\label{def Z}
\end{align}
Here, we set $\z_{i}=\begin{bmatrix}
    \x_{i}\\
    y_i
\end{bmatrix}$ for $i\leq n$ and the last/query token $\z_{n+1}=\begin{bmatrix}
    \x_{n+1}\\
    0
\end{bmatrix}$. Then, given $\Z$, the goal of the model is to predict the correct label $y_{n+1}$ corresponding to $\x_{n+1}$.  For cleaner notation, when it is clear from context, we drop the subscript $n+1$ and set $\x=\x_{n+1},~\z=\z_{n+1}$. Different from the previous work~\citep{ahn2024transformers,mahankali2024one,zhang2024trained,mahdavi2024revisiting} where $(\x_i)_{i=1}^{n+1}$ and $\bt$ are assumed to be independent, our analysis focuses on a more general linear setting that captures the dependency between $(\x_i)_{i=1}^{n+1}$ and $\bt$. 

\noindent\textbf{Model architectures.} To start with, we first review the architectures of both Transformer and state-space model (SSM). Similar to the previous work~\citep{von2023transformers, ahn2024transformers,mahankali2024one,zhang2024trained}\asrnote{Please verify the citation.} and to simplify the model structure, we focus on single-layer models and omit the nonlinearity, e.g., softmax operation and MLP activation,  from the Transformer. Given the input prompt $\Z\in\R^{(n+1)\times (d+1)}$ 
in \eqref{def Z}, which can be treated as a sequence of $(d+1)$-dimensional tokens, the single-layer linear attention $\att$ and H3-like single-layer SSM $\ssm$ are denoted by 
\begin{subequations}\label{att ssm}
\begin{align}
&\att(\Z)=(\Z\WQ\WK^\top\Z^\top)\Z\WV\label{att}\\
&\ssm(\Z)=\left((\Z\WQ)\odot((\Z\WK\odot\Z\WV) \ast\fb)\right)\label{ssm}
\end{align}
\end{subequations}
where $\WK,~\WQ,~\WV\in\R^{(d+1)\times(d+1)}$ denote the key, query and value weight matrices, respectively. In~\eqref{ssm}, the parameter $\fb\in\R^{n+1}$ is a 1-D convolutional filter that mixes tokens. The Hadamard product $\odot$ is the gating mechanism \citep{dauphin2017language} between key and query channels, which is crucial for attention-like feature creation. Thus, \eqref{ssm} is more generally a gated-convolution layer. For $\fb$ only, we use indexing $\fb=[f_0~\dots~f_n]^\top\in\R^{n+1}$ and given any vector $\ab$, denote convolution output $(\ab\ast\fb)_i=\sum_{j=1}^{i} f_{i-j}a_j$. Note that our notation slightly differs from the original H3 model \citep{fu2023hungry} in two ways:
\begin{enumerate}[leftmargin=22pt]
    \item SSMs provide efficient parameterization of $\fb$ which would otherwise grow with sequence length. In essence, H3 utilizes a linear state-space model $\s_i=\A\s_{i-1}+\B u_i$ and $y_i=\Cb\s_i$ with parameters {$(\A\in\R^{d\times d},\B\in\R^{d\times 1},\Cb\in\R^{1\times d})$} from which the filter $\fb$ is obtained via the impulse response $f_i=\Cb\A^i\B$ for $i\geq 0$. Here $d$ is the state dimension and, in practice, $\A$ is chosen to be diagonal. Observe that, setting $d=1$ and $\A=\rho,\Cb=\B=1$, SSM reduces to the exponential smoothing $f_i=\rho^i$ for $i\geq 0$. Thus, H3 also captures the all-ones filter as a special instance. {As we show in Proposition~\ref{lemma:eqv}, this simple filter is optimal under independent data model and exactly imitates linear attention.} Note that, utilizing a filter $\fb$ as in \eqref{ssm} is strictly more expressive than the SSM as it captures all possible impulse responses. 
    \item H3 also applies a shift SSM to the key embeddings to enable the retrieval of the local context around associative recall hits. We opted not to incorporate this shift operator in our model. This is because unless the features of the neighboring tokens are correlated (which is not the case for the typical independent data model), the entry-wise products between values and shifted keys will have zero mean and be redundant for the final prediction.
\end{enumerate}
{We note that we conduct all empirical evaluations with the original H3 model, which displays exact agreement with our theory formalized for \eqref{obj ssm}, further validating our modeling choice.}


\subsection{In-context Linear Estimation}

We will next study the algorithms that can be implemented by the single-layer attention and state-space models. {Through this, we will show that training $\att$ and $\ssm$ with linear ICL data is equivalent to the prediction obtained from one step of optimally-\textit{preconditioned gradient descent} (PGD) and \textit{sample-weighted preconditioned gradient descent} (WPGD), respectively. We will further show that under mild assumption, the optimal sample weighting for $\ssm$ (e.g., $\fb$) is an all-ones vector and therefore, establishing the equivalence among PGD, $\att$, and $\ssm$.}

\noindent \textbf{Background: 1-step gradient descent.} Consider minimizing  squared loss and solving linear regression using one step of PGD and WPGD. Given $n$ samples $(\x_i,y_i)_{i=1}^n$, define 
$$\X=[\x_1~\cdots~\x_n]^\top\in\R^{n\times d}\quad\text{and}\quad \y=[y_1~\cdots~y_n]^\top\in\R^n.$$
Starting from $\bt_0=\zerbb_d$ and letting $\eta=1/2$ be the step size, a single-step GD preconditioned with weights $\W$ returns prediction
\begin{align}
\hat y=\x^\top\W\X^\top\y:=g_{\gd}(\Z)\label{gd},
\end{align}
and a single-step \textit{sample-weighted} GD given weights $\bom\in\R^{n}$ and $\W\in\R^{d\times d}$ returns prediction
\begin{align}
\hat y=\x^\top\W\X^\top(\bom\odot\y):=g_{\wpgd}(\Z),
\label{wpgd}
\end{align}
where $\Z$ is defined in \eqref{def Z} consisting of $\X,\y$ and $\x$. Our goal is to find the optimal $\W$, as well as $\bom$ in \eqref{wpgd} that minimize the population risks defined as follows. 
\begin{subequations}\label{obj gd wpgd}
\begin{align}
\min_{\W}\Lc_{\gd}(\Wc)\quad&\text{where}\quad\Lc_{\gd}(\Wc)=
\E\left[(y-g_{\gd}(\Z))^2\right],\label{obj gd}
\\\min_{\W,\bom}\Lc_{\wpgd}(\Wc)\quad&\text{where}\quad\Lc_{\wpgd}(\Wc)=
\E\left[(y-g_{\wpgd}(\Z))^2\right].\label{obj wpgd}
\end{align}
\end{subequations}
Here, the expectation is over the randomness in $(\x_i,\xi_i)_{i=1}^{n+1}$ and $\bt$, and we use $\Wc$ to represent the set of corresponding trainable parameters. The search spaces for $\bom$ and $\W$ are $\R^n$ and $\R^{d\times d}$, respectively.

As per \eqref{att ssm}, 
given input prompt $\Z\in\R^{(n+1)\times (d+1)}$, either of the underlying models outputs a $(n+1)$-length sequence. Note that the label for the query $\x=\x_{n+1}$ is excluded from the prompt $\Z$. Similar to \citet{ahn2024transformers,mahankali2024one}, we consider a training objective with a causal mask to ensure inputs cannot attend to their own labels and training can be parallelized. Let $\Zm=[\z_1~\dots~\z_n~0]^\top$ be the features 
post-causal masking at time/index $n+1$. Given weights $\WK,\WQ,\WV$ and the filter $\fb$ for $\ssm$, predictions at the query token $\z=\begin{bmatrix}
    \x\\0
\end{bmatrix}$ take the following forms following sequence-to-sequence mappings in \eqref{att ssm}:
\begin{align*}
&g_{\att}(\Z)=(\z^\top\WQ\WK^\top\Zm^\top)\Zm\WV\vb,\\
&g_{\ssm}(\Z)=\left((\z^\top\WQ)^\top\odot ((\Zm\WK\odot\Zm\WV) \ast\fb)_{n+1}\right)\vb,
\end{align*}
where $\vb\in\R^{d+1}$ is the linear prediction head and $((\Zm\WK\odot\Zm\WV) \ast\fb)_{n+1}$ returns the last row of the convolution output. Note that SSM can implement the mask by setting $f_0=0$. 
Now consider the meta learning setting 
and select loss function to be the squared loss, same as in \eqref{obj gd wpgd}. Thus, the objectives for both models take the following forms.
\begin{subequations}\label{obj att ssm}
\begin{align}
\min_{\WK,\WQ,\WV,\vb}\Lc_{\att}(\Wc)\quad&\text{where}\quad\Lc_{\att}(\Wc)=
\E\left[(y-g_{\att}(\Z))^2\right]\label{obj att},\\
\min_{\WK,\WQ,\WV,\vb,\fb}\Lc_{\ssm}(\Wc)\quad&\text{where}\quad\Lc_{\ssm}(\Wc)=
\E\left[(y-g_{\ssm}(\Z))^2\right].\label{obj ssm}
\end{align}
\end{subequations}
Here, similarly, the expectation subsumes the randomness of $(\x_i,\xi_i)_{i=1}^{n+1}$ and $\bt$ and $\Wc$ represents the set of trainable parameters. The search space for matrices $\WK$, $\WQ$, $\WV$ is $\R^{(d+1)\times(d+1)}$, for head $\vb$ is $\R^{d+1}$, and for $\fb$ is $\R^{n+1}$.

Note that for all the optimization methods (c.f.~\eqref{obj gd wpgd}, \eqref{obj att ssm}), to simplify the analysis, we train the models without capturing additional bias terms. Therefore, in the following, we introduce the centralized data assumptions such that the models are trained to make unbiased predictions.

To begin with, 
a cross moment of random variables is defined as the expectation of a monomial of these variables, with the order of the cross moment being the same as order of the monomial. For example, $\E[\x^\top\W\bt]$ is a sum of cross-moments of order 2. Then, it motivates the following data assumptions. 
\begin{assumption}\label{assum:odd x beta}
    All cross moments of the entries of $(\x_i)_{i=1}^{n+1}$ and $\bt$ with odd orders are zero. 
\end{assumption}


\begin{assumption}\label{assume:noise}
    The label noise $(\xi_i)_{i=1}^{n+1}$ are independent of $(\x_i)_{i=1}^{n+1}$ and $\bt$, and their cross moments with odd orders are zero. 
\end{assumption}
Note that compared to \cite{ahn2024transformers,mahankali2024one,zhang2024trained}, Assumption~\ref{assum:odd x beta} is more general which also subsumes the dependent distribution settings. 
In this work, we consider the following three linear models (omitting noise) satisfying Assumption~\ref{assum:odd x beta}. Let $\bSi_{\bt},\bSi_{\x}\in\R^{d\times d}$ represent the task and feature covariance matrices for independent data, and let $0\leq\alpha\leq1$ be the correlation level when considering data dependency. More specific discussions are deferred to Section~\ref{sec:main}.
 %
\begin{enumerate}[label=$\bullet$,leftmargin=20pt]
    \item Independent task and data: $\bt\sim\Nc(0,\bSi_{\bt}),~\x_i\sim\Nc(0,\bSi_{\x}),~~\text{for all}~~1\leq i\leq n+1$. 
\item Retrieval augmented generation: $\bt,\x\sim\Nc(0,\Iden_d),~ \x_i{\bgl\x}\sim\Nc(\alpha\x,(1-\alpha^2)\Iden_d),~~\text{for all}~~1\leq i\leq n$.
\item Task-feature alignment: $\bt\sim\Nc(0,\Iden_d),~\x_i{\bgl \bt}\sim\Nc(\alpha\bt,\Iden_d),~~\text{for all}~~1\leq i\leq n+1.$
\end{enumerate}

Next, we introduce the following result which establishes the equivalence among optimizing 1-layer linear attention (c.f.~\eqref{obj att}), 1-layer H3 (c.f.~\eqref{obj ssm}), and 1-step gradient descent (c.f. \eqref{obj gd wpgd}).

{
\begin{proposition}\label{lemma:eqv}Suppose Assumptions~\ref{assum:odd x beta} and \ref{assume:noise} hold. Consider the objectives as defined in \eqref{obj gd  wpgd} and \eqref{obj att ssm}, and let $\Lc_{\gd}^\st,~\Lc_{\wpgd}^\st,~\Lc^\st_\att$, and $\Lc^\st_{\ssm}$ be their optimal risks, respectively. Then, 
\[
\Lc^\st_{\gd}=\Lc_{\att}^\st
\quad\text{and}\quad
\Lc^\st_{\wpgd}=\Lc_{\ssm}^\st.
\]
{Additionally, if the examples $(\x_i,y_i)_{i=1}^n$ follow the same distribution and are conditionally independent given $\x,\bt$, then SSM/H3 can achieve the optimal loss using the all-ones filter and $\Lc^\st_{\gd}=\Lc_{\ssm}^\st$.}
\end{proposition}
}

We defer the proof to Appendix~\ref{app:proof eqv}. 
Proposition~\ref{lemma:eqv} 
establishes that analyzing the optimization landscape of ICL for both single-layer linear attention and the H3 model can be effectively reduced to examining the behavior of a one-step 
PGD algorithm. {Notably, under the independent, RAG and task-feature alignment data settings discussed above, examples $(\x_i,y_i)_{i=1}^n$ are independently sampled given $\x$ and $\bt$, and we therefore conclude that $\Lc_{\gd}^\st=\Lc_{\att}^\st=\Lc_{\ssm}^\st$.} Leveraging this result, the subsequent section of the paper concentrate on addressing \eqref{obj gd}, taking into account various linear data distributions. 

{
While Proposition~\ref{lemma:eqv} demonstrates the equivalence of optimal losses, we also study the uniqueness and equivalence of optimal prediction functions. To this end, we analyze the strong convexity of $\Lc_{\gd}(\Wc)$ and derive the subsequent lemmas.
\begin{lemma}\label{lemma cvx 1}
    {Suppose Assumption~\ref{assume:noise} holds} and let $\bxi=[\xi_1~\xi_2~\cdots~\xi_n]^\top$. Then the loss $\Lc_{\gd}(\Wc)$ in \eqref{obj gd} is strongly-convex if and only if $\E[(\x^\top\W\X^\top\X\bt)^2]+\E[(\x^\top\W\X^\top\bxi)^2]$ is strongly-convex. Additionally, let $g_{\gd}^\st$, $g^\st_\att$ be the optimal prediction functions of \eqref{obj gd} and \eqref{obj att}. Then under the conditions of Assumptions~\ref{assum:odd x beta} and \ref{assume:noise}, and the strong convexity,  $g^\st_{\gd}=g_{\att}^\st$. 
\end{lemma}
\begin{lemma}\label{lemma cvx 2}
   {Suppose that the label noise $(\xi_i)_{i=1}^n$ are i.i.d., zero-mean, variance $\sigma^2$ and independent of everything else}, and that there is a decomposition $\x=\x_1+\x_2$, $\X=\X_1+\X_2$, and $\bt=\bt_1+\bt_2$ such that either of the following holds
    \begin{enumerate}[label=$\bullet$,leftmargin=20pt]
    \item $\sigma>0$, and $(\x_1,\X_1)$ have full rank covariance and are independent of each other and $(\x_2,\X_2)$.
    \item $(\x_1,\bt_1,\X_1)$ have full rank covariance and are independent of each other and $(\x_2,\bt_2,\X_2)$.
    \end{enumerate}
    Then, the loss $\Lc_{\gd}(\Wc)$ in \eqref{obj gd} is strongly-convex.
\end{lemma}
As mentioned above, in this work, we study three specific linear models: with general independent, RAG-related, and task-feature alignment data. Note that for all the three cases, according to Proposition~\ref{lemma:eqv}, we have $\Lc^\st_{\gd}=\Lc^\st_{\att}=\Lc^\st_{\ssm}$. Additionally, the second claim in Lemma~\ref{lemma cvx 2} holds, and $\Lc_{\gd}(\Wc)$ is strongly convex. Therefore, following Lemma~\ref{lemma cvx 1}, we have $g^\st_{\gd}=g^\st_\att$. Thanks to the equivalence among PGD, $\att$, and $\ssm$, in the next section, we focus on the solution of objective~\eqref{obj gd} under different scenarios, which will reflect the optimization landscapes of $\att$ and $\ssm$ models.}


\section{Main Results}\label{sec:main}
In light of Proposition~\ref{lemma:eqv}, optimizing a single layer linear-attention or H3 model is equivalent to solving the objective \eqref{obj gd}. Therefore, in this section, we examine the properties of the one-step 
PGD in \eqref{obj gd}. To this end, we consider multiple problem settings, including distinct data distributions and low-rank training. The latter refers to the scenario where the key and query matrices have rank restrictions, e.g., $\WK,\WQ\in\R^{(d+1)\times r}$, as well as LoRA-tuning when adapting the model under distribution shift.  
\subsection{Analysis of Linear Data Models}
We first consider the standard independent data setting. We will then examine correlated designs.

\noindent\textbf{Independent data model.} Let $\bSi_{\x}$ and $\bSi_{\bt}$ be the covariance matrices of the input feature and task vectors, respectively, and $\sigma\geq0$ be the noise level. We assume
\begin{align}
\bt\sim\Nc(0,\bSi_{\bt}),\quad\x_i\sim\Nc(0,\bSi_{\x}),\quad\xi_i\sim\Nc(0,\sigma^2),\quad 1\leq i\leq n+1\label{data ind}
\end{align}
and the label is obtained via $y_i=\x_i^\top\bt+\xi_i$. Our following result characterizes  the optimal solution of \eqref{obj gd}. Note that the data generated from \eqref{data ind} satisfies the conditions in Proposition~\ref{lemma:eqv}. Therefore, the same results can be applied to both linear-attention and H3 models.
\begin{theorem}\label{thm:independent}
    Consider independent linear data as defined in \eqref{data ind}, and suppose the covariance matrices $\bSi_{\x},\bSi_{\bt}$ are full rank. Recap the objective from \eqref{obj gd} and let $\W_\st:=\arg\min_{\W}\Lc_{\gd}(\W)$, and $\Lc_\st=\Lc_{\gd}(\W_\st)$. Additionally, let $\bSi=\bSi^{1/2}_{\x}\bSi_{\bt}\bSi_{\x}^{1/2}$ and $M=\tr{\bSi}+\sigma^2$. Then $\Ws$ and $\Lcs$ satisfy
    \begin{align}
    \W_\st=\bSi_{\x}^{-1/2}\Wb_\st\bSi_{\x}^{-1/2}\qquad\text{and}\qquad
    \Lc_\st=M-n\tr{\bSi\Wb_\st}\label{formula ind},
    \end{align}
    where we define $\Wb_\st=\left((n+1)\Iden_d+M\bSi^{-1}\right)^{-1}$.
\end{theorem}

\begin{corollary}\label{corol:iid}
    Consider noiseless i.i.d. linear data where $\bSi_{\x}=\bSi_{\bt}=\Iden_d$ and $\sigma=0$. Then, the objective in \eqref{obj gd} returns
    \[
    \W_\st=\frac{1}{n+d+1}\Iden_d\qquad\text{and}\qquad\Lcs=d-\frac{nd}{n+d+1}.
    \]
\end{corollary}
See Appendix~\ref{app:independent} for proofs. 
Note that Theorem~\ref{thm:independent} is consistent with prior work \citep[Theorem 1]{ahn2024transformers} when specialized to isotropic task covariance, i.e.,~$\bSi_{\bt}=\Iden_d$. However, their result is limited as the features and task are assumed to be independent. This prompts us to ask: \emph{What is the optimization landscape with correlated in-context samples?} Toward this, we consider the following RAG-inspired and task-feature alignment models, 
where Assumptions~\ref{assum:odd x beta} and \ref{assume:noise} continue to hold and Proposition~\ref{lemma:eqv} applies. 

\noindent\textbf{Retrieval augmented generation.} To provide a statistical model of the practical RAG approaches, given the query vector $\x_{n+1}=\x$, we propose to draw ICL demonstrations that are similar to $\x$ with the same shared task vector $\bt$. Modeling feature similarity through the cosine angle, RAG should sample the ICL examples $\x_i$, $i\leq n$, from the original feature distribution conditioned on the event $\text{cos}(\x_i,\x)\geq \alpha$ where $\alpha$ is the similarity threshold. 
As an approximate proxy, under the Gaussian distribution model, we assume that $\bt\sim\Nn(0,\Iden_d),~\x\sim\Nn(0,\Iden_d)$ and that RAG samples $\alpha$-correlated demonstrations $(\x_i,y_i)_{i=1}^n$ as follows:
\begin{align}
\x_i\bgl\x\sim\Nn(\alpha\x,(1-\alpha^2)\Iden_d),\quad\xi_i\sim\Nc(0,\sigma^2)\quad\text{and}\quad y_i=\x_i^\top \bt+\xi_i,\quad 1\leq i\leq n.\label{data rag}
\end{align}
Note that the above normalization ensures that the marginal feature distribution remains $\Nn(0,\Iden_d)$. 
The full analysis of RAG is provides in Appendix~\ref{app:rag}. Specifically, when we carry out the analysis by assuming $\alpha=\order{1/\sqrt{d}}$ and $d/n=\order{1}$ where $\order{\cdot}$ denotes proportionality, our derivation leads to the following result:
{
\begin{theorem}\label{thm:rag}
    Consider linear model as defined in \eqref{data rag}. Recap the objective from \eqref{obj gd} and let $\W_\st:=\arg\min_{\W}\Lc_{\gd}(\W)$, and $\Lc_\st=\Lc_{\gd}(\W_\st)$. Additionally, let $\kappa=\alpha^2d+1$ and suppose $\alpha=\order{1/\sqrt{d}}$, $d/n=\order{1}$ and $d$ is sufficiently large. Then $\Ws$ and $\Lcs$ have approximate forms
    \begin{align}
    \W_\st\approx\frac{1}{\kappa n+d+\sigma^2}\Iden_d\qquad\text{and}\qquad
\Lc_\st\approx d+\sigma^2-\frac{\kappa nd}{\kappa n+d+\sigma^2}.\label{formula rag}
\end{align}
\end{theorem}
}

Here, \eqref{formula rag} is reminiscent of Corollary~\ref{corol:iid} and has a surprisingly clean message. Observe that, $\alpha^2d+1$ is the dominant multiplier ahead of $n$ in both equations. Thus, we deduce that, RAG model follows the same error bound as the independent data model, however, its sample size is amplified by a factor of $\alpha^2d+1$. $\alpha=0$ reduces to the result of Corollary~\ref{corol:iid} whereas we need to set $\alpha=\order{1/\sqrt{d}}$ for constant amplification. When $\alpha=1$, RAG achieves the approximate risk $\Lcs\approx 2+\sigma^2$, where the constant bias is due to the higher order moments (e.g., the $4$'th and $6$'th moments) of the standard Gaussian distribution.  As $d$ increases, the normalized loss $\Lc_\st/d\to 0$. The full analysis of its optimal solution $\W_\st$ and loss $\Lc_\st$ are deferred Theorem~\ref{thm:rag exact} in Appendix~\ref{app:rag}. 


{
\noindent\textbf{Task-feature alignment.} We also consider another dependent data setting where task and feature vectors are assumed to be correlated. {This dataset model has the following motivation: In general, an LLM can generate any token within the vocabulary. However, once we specify the task (e.g.~domain of the prompt), the LLM output becomes more deterministic and there are much fewer token candidates. For instance, if the task is ``Country'', ``France'' is a viable output compared to ``Helium'' and vice versa when the task is ``Chemistry''. Formally speaking, this can be formalized as the input $\x$ having a diverse distribution whereas it becomes more predictable conditioned on $\bt$. Therefore, it can be captured through a linear model by making the conditional covariance of $\x\bgl\bt$ to be approximately low-rank. This formalism can be viewed as a \emph{spectral alignment} between input and task, which is also well-established in deep learning both empirically and theoretically \citep{li2020gradient,arora2019fine,canatar2021spectral,cao2019towards}. Here, we consider such a setting where} the shared task vector is sampled as standard Gaussian distribution $\bt~\sim\Nc(0,\Iden_d)$ and letting $\kappa=\alpha^2d+1$, we sample the $\alpha$-correlated ICL demonstrations $(\x_i,y_i)_{i=1}^{n+1}$ as follows:}
\begin{align}
{\x_i\bgl\bt} \sim\Nn(\alpha\bt,\Iden_d),\quad\xi_i\sim\Nc(0,\sigma^2)\quad\text{and}\quad y_i={\kappa^{-1/2}}\x_i^\top \bt+\xi_i,\quad 1\leq i\leq n+1.\label{data tfa}
\end{align}
{Above, ${\kappa^{-1/2}}$ is a normalization factor to ensure that label variance remains invariant to $\alpha$.} To keep the exposition cleaner, we defer the full analysis of its optimal solution $\W_\st$ and loss $\Lc_\st$ to Theorem~\ref{thm:feature app} in Appendix~\ref{app:task feature}. Similar to the RAG setting, by assuming $\alpha=\order{1/\sqrt{d}}$ and $d/n=\order{1}$, we obtain the following results for the optimal parameter and risk.
\begin{theorem}\label{thm:feature}
    Consider linear model as defined in \eqref{data tfa}. Recap the objective from \eqref{obj gd} and let $\W_\st:=\arg\min_{\W}\Lc_{\gd}(\W)$, and $\Lc_\st=\Lc_{\gd}(\W_\st)$. Additionally, given $\kappa=\alpha^2d+1$ and suppose $\alpha=\order{1/\sqrt{d}}$, $d/n=\order{1}$ and $d$ is sufficiently large. Then $\Ws$ and $\Lcs$ have approximate forms
    \begin{align}
    \W_\st\approx\frac{1}{\kappa n+(d+\sigma^2)/\kappa}\Iden_d\qquad\text{and}\qquad\Lc_\st\approx  d+\sigma^2-\frac{\kappa nd}{\kappa n+(d+\sigma^2)/\kappa}.
    \label{formula task feature}
    \end{align}
\end{theorem}
Similar to \eqref{formula rag}, \eqref{formula task feature} contains $\kappa=\alpha^2+1$ multiplier ahead of $n$, which reduces the in-context sample complexity and setting $\alpha=0$ reduces to the results of Corollary~\ref{corol:iid}. 

\subsection{Low-rank Parameterization and LoRA}\label{sec:low rank}

In this section, we investigate training low-rank models, which assume $\WK,\WQ\in\R^{(d+1)\times r}$ where $r$ is the rank restriction. 
Equivalently, we consider objective~\eqref{obj gd} under condition $\rank{\W}=r$.

\begin{lemma}\label{low rank attn}
Consider independent linear data as defined in \eqref{data ind}. Recap the objective from \eqref{obj gd} and enforce $\rank{\W}\leq r$ and $\W^\top=\W$. Let $\bSi=\bSi^{1/2}_{\x}\bSi_{\bt}\bSi_{\x}^{1/2}$ and $M=\tr{\bSi}+\sigma^2$. 
Denoting $\la_i$ to be the $i$'th largest eigenvalue of $\bSi$, we have that 
\begin{align}
\min_{\rank{\W}\leq r,\W=\W^\top}\Lc(\W)=M-\sum_{i=1}^r\frac{n\la_i^2}{(n+1)\la_i+M}.\label{formula low rank}
\end{align}
\end{lemma}
Note that $\tr{\bSi}=\sum_{i=1}^d\la_i$. Removing the rank constraint and considering noiseless data setting, this reduces to the following optimal risk
 %
$\Lcs=\sum_{i=1}^d\frac{\la_i+M}{n+1+M/\la_i}$. See Appendix~\ref{app:low-rank} for more details.

\noindent\textbf{Impact of LoRA:} Based on the above lemma, we consider the impact of LoRA for adapting the pretrained model to a new task distribution under jointly-diagonalizable old and new eigenvalues of $\bSi,~\bSi^{new}$, $(\la_i)_{i=1}^d,~(\la_i^{new})_{i=1}^d$. 
Consider adapting LoRA matrix to the combined key and value weights in attention, which reflects minimizing the population loss $\tilde\Lc(\WLR):=\Lc(\W+\WLR)$ in \eqref{obj gd} with fixed $\W$. 
Suppose $\tr{\bSi}=\tr{\bSi^{new}}=M$, $\sigma=0$ and $\W$ is jointly diagonalizable with $\bSi,~\bSi^{new}$, then LoRA's risk is upper-bounded by
\begin{align}
\min_{\rank{\WLR}\leq r}\tilde\Lc(\W_{lora})\leq
\min_{|\Ic|\leq r,\Ic\subset[d]}\left(\sum_{i\not\in\Ic}\frac{\lambda_i+M}{n+1+M/\la_i}+\sum_{i\in\Ic}\frac{\la^{new}_i+M}{n+1+M/\la^{new}_i}\right). \label{formula lora}
\end{align}
Note that, the right hand side is provided assuming the optimal LoRA-updated model $\WLR$ is also jointly diagonalizable with covariances $\bSi,~\bSi^{new}$, and $\W$.

\begin{figure}[!t]
\begin{subfigure}{0.325\linewidth}    
\centering
\includegraphics[height=0.7\linewidth]{figs/test_iid.pdf}
\caption{$\bSi_{\x}=\bSi_{\bt}=\Iden_d$ and $\sigma=0$}\label{fig:iid}
\end{subfigure}
\begin{subfigure}{0.325\linewidth}    
    \centering
    \includegraphics[height=.7\columnwidth]{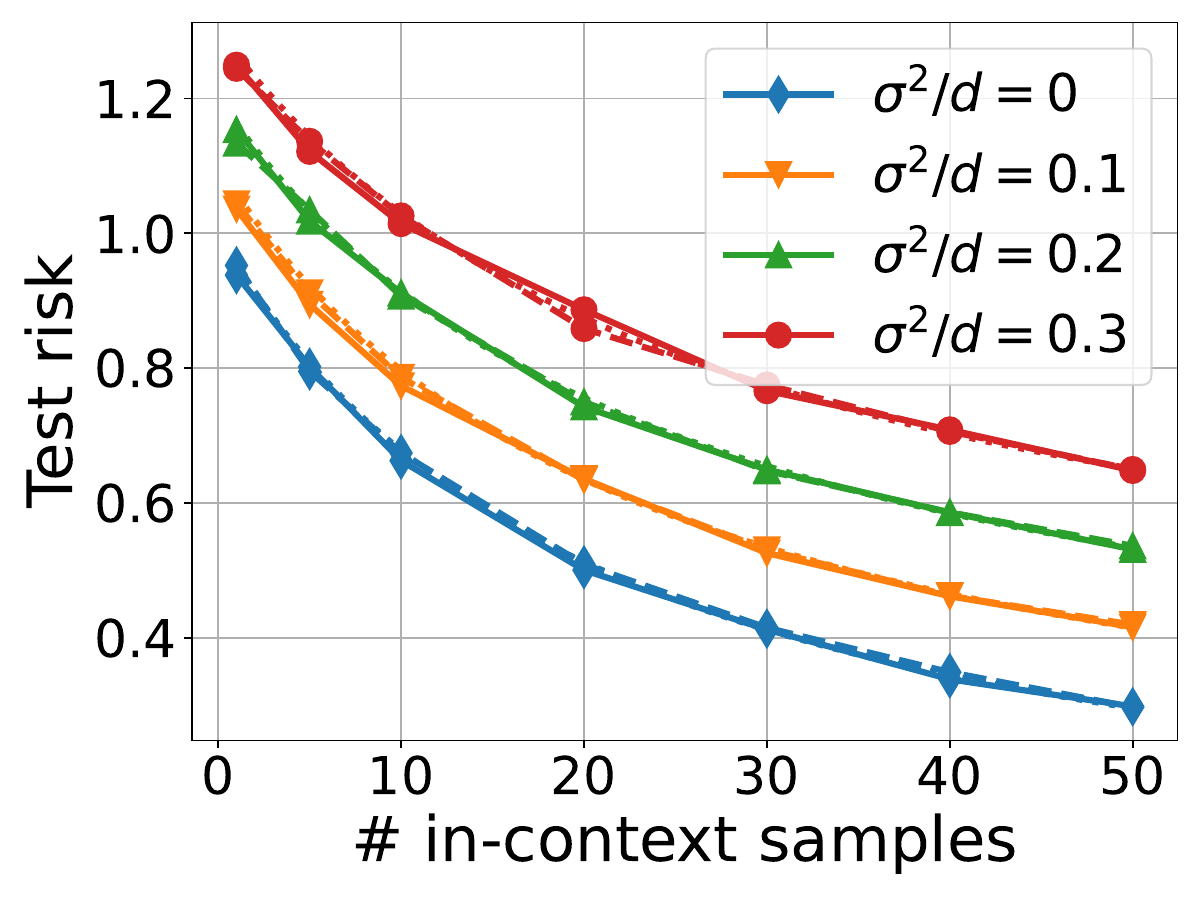}
\caption{Noisy label}\label{fig:noise}
\end{subfigure}
\begin{subfigure}{0.325\linewidth}    
\centering
\includegraphics[height=.7\columnwidth]{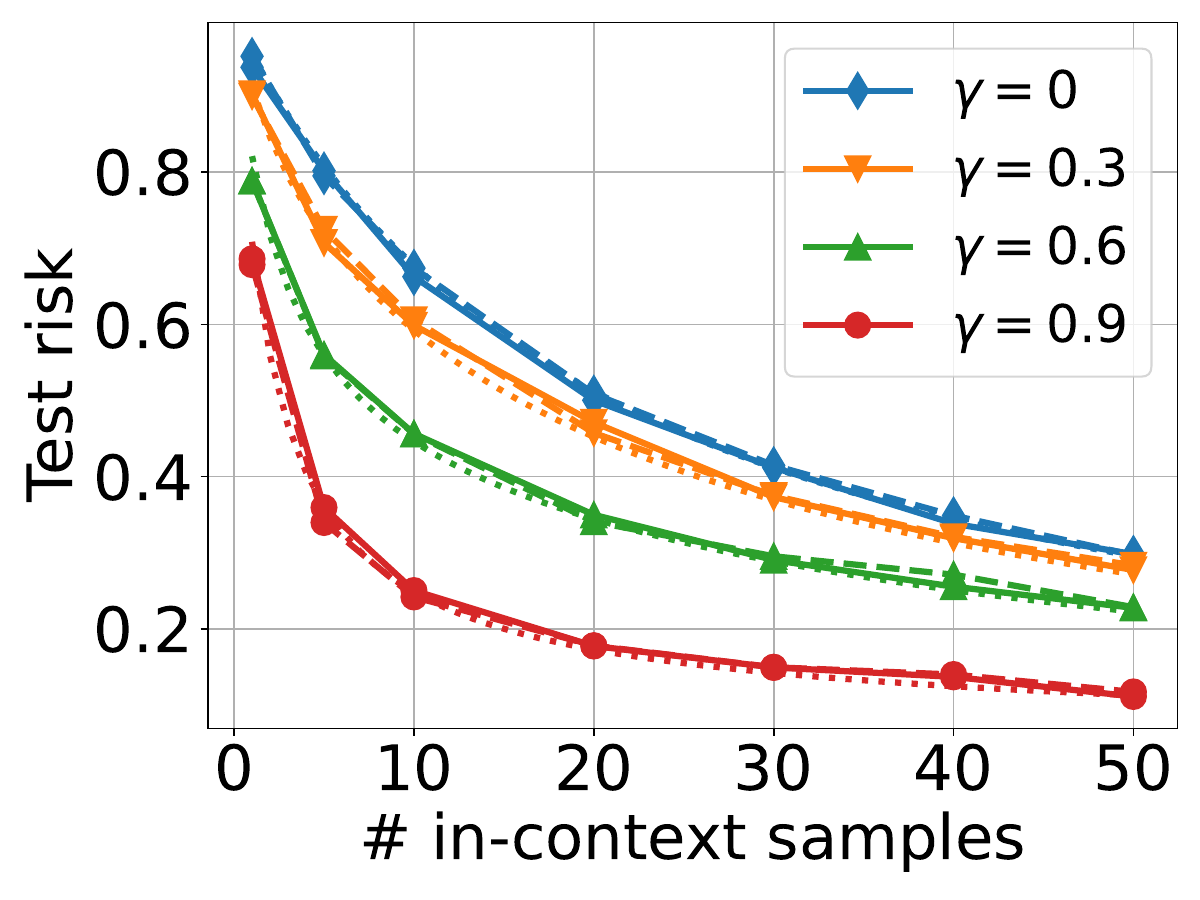}
\caption{Non-isotropic task}\label{fig:cov}
\end{subfigure}
\caption{Empirical evidence validates Theorem~\ref{thm:independent} and Proposition~\ref{lemma:eqv}. We train 1-layer linear attention and H3 models with prompts containing independent demonstrations following a linear model, and dotted curves are the theory curves following Eq.~\eqref{formula ind}. \textbf{(a):} We consider noiseless i.i.d. setting where $\bSi_{\x}=\bSi_{\bt}=\Iden_{d}$ and $\sigma=0$, with results presented in red (attention) and blue (H3) solid curves. \textbf{(b):} We conduct noisy label experiments by choosing $\sigma\neq0$. \textbf{(c):} Consider non-isotropic task by setting $\bSi_{\bt}=\gamma\onebb\onebb^\top+(1-\gamma)\Iden_d$. Solid and dashed curves in (b) and (c) represent attention and H3 results, respectively. The alignments in (a), (b) and (c) show the equivalence between attention and H3, validating Theorem~\ref{thm:independent} and Proposition~\ref{lemma:eqv}. More experimental details are discussed in Section~\ref{sec exp}.}\label{fig:independent}
\end{figure}

\section{Experiments}\label{sec exp}


We now conduct synthetic experiments 
to support our theoretical findings and further explore the behavior of different models of interest under different conditions. The experiments are designed to investigate various scenarios, including independent data, retrieval-augmented generation (RAG), task-feature alignment, low-rank parameterization, and LoRA adaption. 

\textbf{Experimental setting.} We train 1-layer attention and H3 models for solving the linear regression ICL. As described in Section~\ref{sec:setup}, we consider meta-learning setting where task parameter $\bt$ is randomly generated for each training sequence. In all experiments, we set the dimension $d=20$. 
Depending on the in-context length ($n$), different models are trained to make in-context predictions.  We train each model for $10000$ iterations with batch size $128$ and Adam optimizer with learning rate $10^{-3}$. Since our study focuses on the optimization landscape, and experiments are implemented via gradient descent, we repeat $20$ model trainings from different initialization and results are presented as the minimal test risk among those $20$ trails. 
{In all the plots, theoretical predictions are obtained via the corresponding formulae presented in Section~\ref{sec:main} and the test risks are normalized by the dimension $d$.}

\begin{figure}[!t]
\centering
\centering
\begin{subfigure}{0.244\linewidth}
    \includegraphics[height=0.75\linewidth]{figs/test_rag_fixed.pdf}
\caption{RAG}\label{fig:rag loss}
\end{subfigure}
\begin{subfigure}{0.244\linewidth}
    \includegraphics[height=0.75\linewidth]{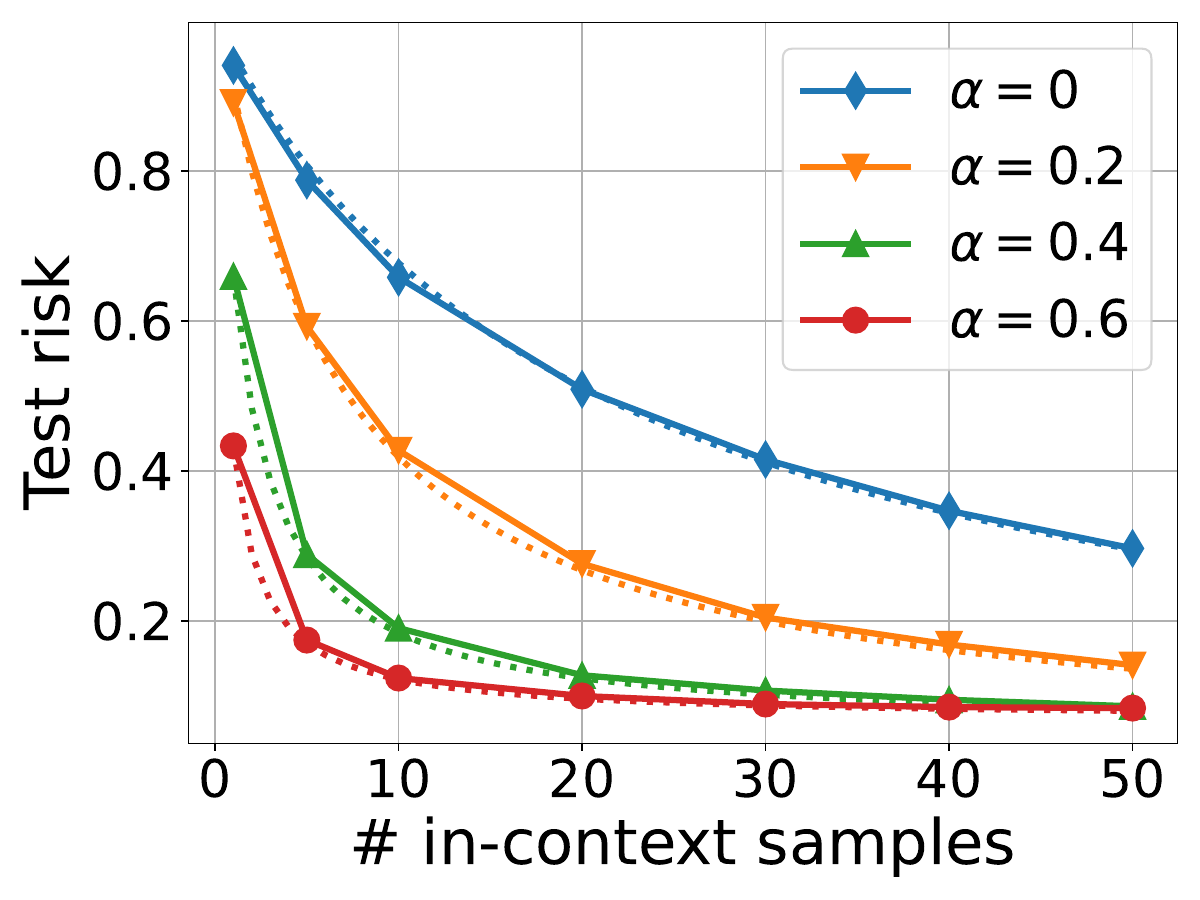}
\caption{Task-feature alignment}\label{fig:corr feature}
\end{subfigure}
\begin{subfigure}{0.244\linewidth}
    \includegraphics[height=0.75\linewidth]{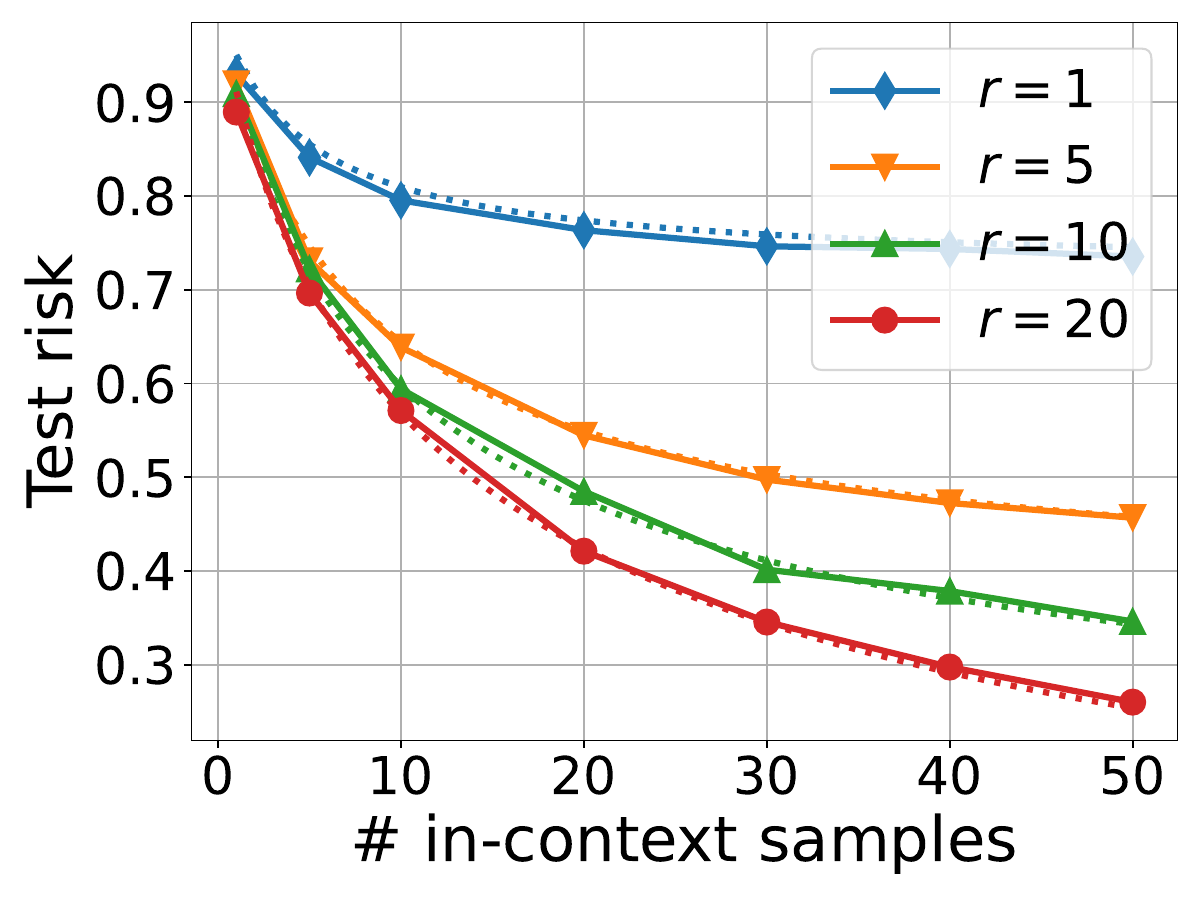}
\caption{Low-rank attention}\label{fig:low rank d20}
\end{subfigure}
\centering
\centering
\begin{subfigure}{0.244\linewidth}    
\centering
    \includegraphics[height=.75\columnwidth]{figs/lora_d20_2.pdf}
\caption{LoRA adaptation}\label{fig:lora}
\end{subfigure}
\caption{Distributional alignment and low-rank parameterization experiments. 
\textbf{(a)} and \textbf{(b)} show the ICL results using data generated via \eqref{data rag} and \eqref{data tfa}, respectively, by changing $\alpha$ from $0$ to $0.6$. In \textbf{(c)}, we train low-rank linear attention models by setting $\W_k,\W_q\in\R^{(d+1)\times r}$ and in \textbf{(d)}, we apply the low-rank LoRA adaptor, $\WLR:=\W_{\text{up}}\W_{\text{down}}^\top$ where $\W_{\text{up}},\W_{\text{down}}\in\R^{(d+1)\times r}$, to pretrained linear attention models and adjust the LoRA parameters under different task distribution. Solid and dotted curves correspond to the linear attention and  theoretical results (c.f.~Section~\ref{sec:main}), respectively, and the alignments validate our theorems in Section~\ref{sec:main}.  
More experimental details are discussed in Section~\ref{sec exp}.}
\end{figure}

\noindent$\bullet$ \textbf{Equivalence among $\Lc^\st_{\gd}$, $\Lc_{\att}^\st$ and $\Lc_{\ssm}^\st$ (Figure~\ref{fig:independent}).}
To verify Proposition~\ref{lemma:eqv} as well as Theorem~\ref{thm:independent}, we run random linear regression instances where in-context samples are generated obeying \eqref{data ind}. Fig.~\ref{fig:iid} is identical to Fig.~\ref{intro fig:iid} where we set $\bSi_{\x}=\bSi_{\bt}=\Iden_d$ and $\sigma=0$. In Fig.~\ref{fig:noise}, set $\bSi_{\x}=\bSi_{\bt}=\Iden$ and vary noise level $\sigma^2$ from $0$ to $0.3\times d$. In Fig.~\ref{fig:cov}, we consider noiseless labels, $\sigma=0$, isotropic feature distribution $\bSi_{\x}=\Iden_d$ and set task covariance to be $\bSi_{\bt}=\gamma\onebb\onebb^\top+(1-\gamma)\Iden_d$ by choosing $\gamma$ in $\{0, 0.3, 0.6, 0.9\}$. Note that in Fig.~\ref{fig:cov}, we train a sufficient number of models (greater than $20$) to ensure the optimal model is obtained. In all the figures, solid and dashed curves correspond to the ICL results from training 1-layer $\att$ and $\ssm$ models, respectively, and dotted curves are obtained from \eqref{formula ind} in Theorem~\ref{thm:independent}. The alignment of solid, dashed and dotted curves validates our Proposition~\ref{lemma:eqv} and Theorem~\ref{thm:independent}.

\noindent$\bullet$ \textbf{Distributional alignment experiments (Figs.~\ref{fig:rag loss}\&\ref{fig:corr feature}).}
In Figs.~\ref{fig:rag loss} and \ref{fig:corr feature}, we generate RAG and task-feature alignment data following \eqref{data rag} and \eqref{data tfa}, respectively, by setting $\sigma=0$ and varying $\alpha$ from $0$ to $0.6$. Attention training results are displayed in solid curves, and we generate theory curve (dotted) via the $\Lcs$ formula as described in \eqref{formula rag app 1} in Appendix~\ref{app:rag} and \eqref{formula task feature app 1} in Appendix~\ref{app:task feature}.
The empirical alignments corroborate Theorems~\ref{thm:rag exact} and \ref{thm:feature app}, further confirming that Proposition~\ref{lemma:eqv} is applicable to a broader range of real-world distributional alignment data.

\noindent$\bullet$ \textbf{Low-rank (Fig.~\ref{fig:low rank d20}) and LoRA (Fig.~\ref{fig:lora}) experiments.} We also run simulations to verify our theoretical findings in Section~\ref{sec:low rank}. Consider the independent data setting as described in \eqref{data ind}. In Fig.~\ref{fig:low rank d20}, we set 
$\bSi_{\x}=\Iden_d$, $\sigma=0$ and task covariance to be diagonal with diagonal entries $c[1~{2}^{-1}~\cdots~d^{-1}]^\top$ for some normalization constant $c=d/\sum_{i=1}^di^{-1}$, and parameterize the attention model using matrices $\WK,\WQ\in\R^{(d+1)\times r}$ and vary $r$ across the set $\{1,5,10,20\}$. Results show that empirical (solid) and theoretical (dotted, c.f.~\eqref{formula low rank}) curves overlap. In Fig.~\ref{fig:lora}, we implement two phases of training. \textit{Phase 1:} Setting $\bSi_{\x}=\bSi_{\bt}=\Iden_d$ and $\sigma=0$, we pretrain the model with full rank parameters and obtain weights $\WKh,\WQh,\WVh\in\R^{(d+1)\times (d+1)}$. \textit{Phase 2:} We generate new examples with task covariance $\bSi_{\bt}$ being a diagonal matrix with diagonal entries $c'[{2}^{-1}~2^{-2}~\cdots~2^{-d}]^\top$ for some normalization constant $c'=d/\sum_{i=1}^d2^{-i}$. Given the rank restriction $r$, we train additional LoRA parameters $\W_{\text{up}},\W_{\text{down}}\in\R^{(d+1)\times r}$
where $\WLR:=\W_{\text{up}}\W_{\text{down}}^\top$ and \eqref{att} becomes $\att(\Z)=(\Z(\WQh\WKh^\top+\W_{\text{up}}\W_{\text{down}}^\top)\Z^\top)\Z\WVh$. Fig.~\ref{fig:lora} presents the results after two phases of training where dotted curves are drawn from the right hand side of \eqref{formula lora} directly. Here, note that since $\bSi,\bSi^{new}$ are diagonal, the right hand side of \eqref{formula lora} returns the exact optimal risk of  LoRA and the alignments verify it. 

\begin{figure}[!t]
\centering
\begin{subfigure}{0.244\linewidth}
    \includegraphics[height=0.75\linewidth]{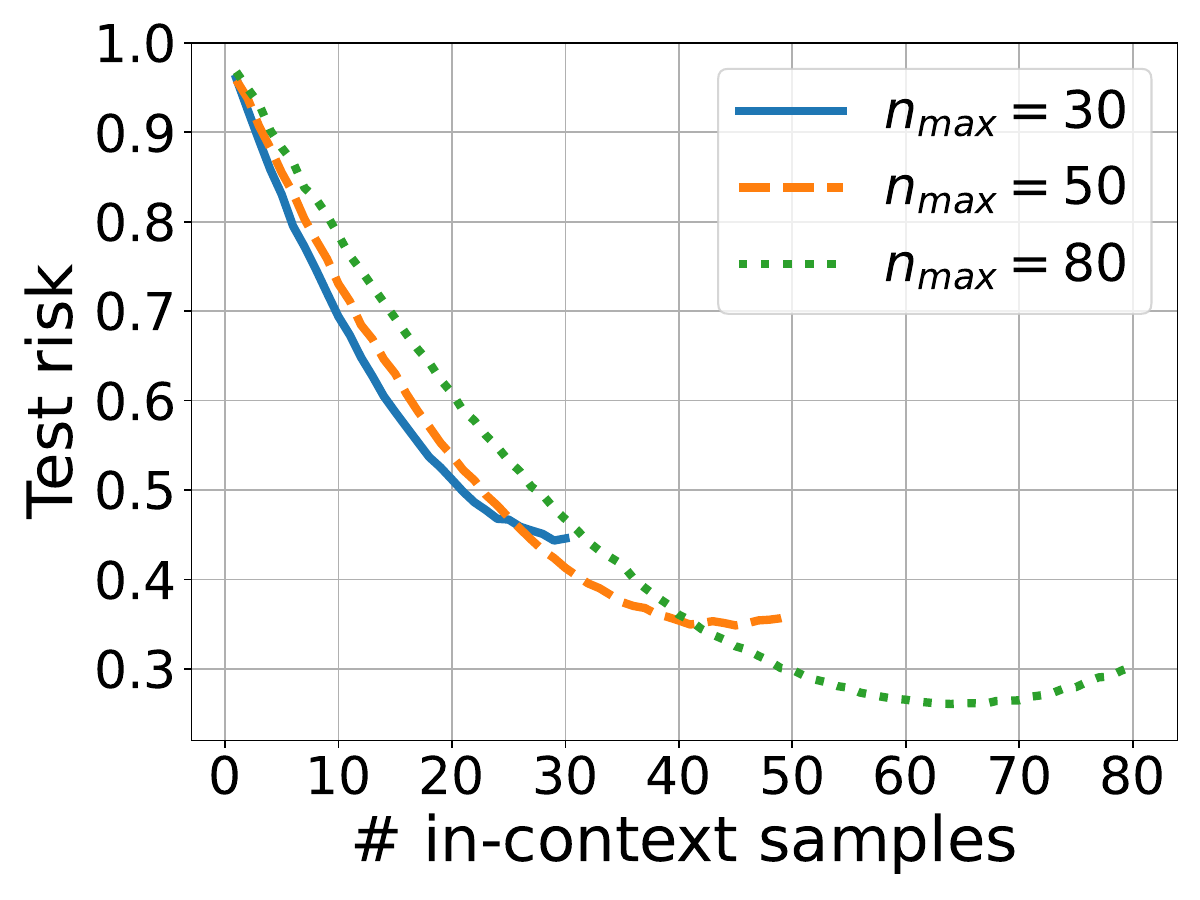}
\caption{Linear attention}\label{fig:avg att}
\end{subfigure}
\begin{subfigure}{0.244\linewidth}    
    \includegraphics[height=.75\columnwidth]{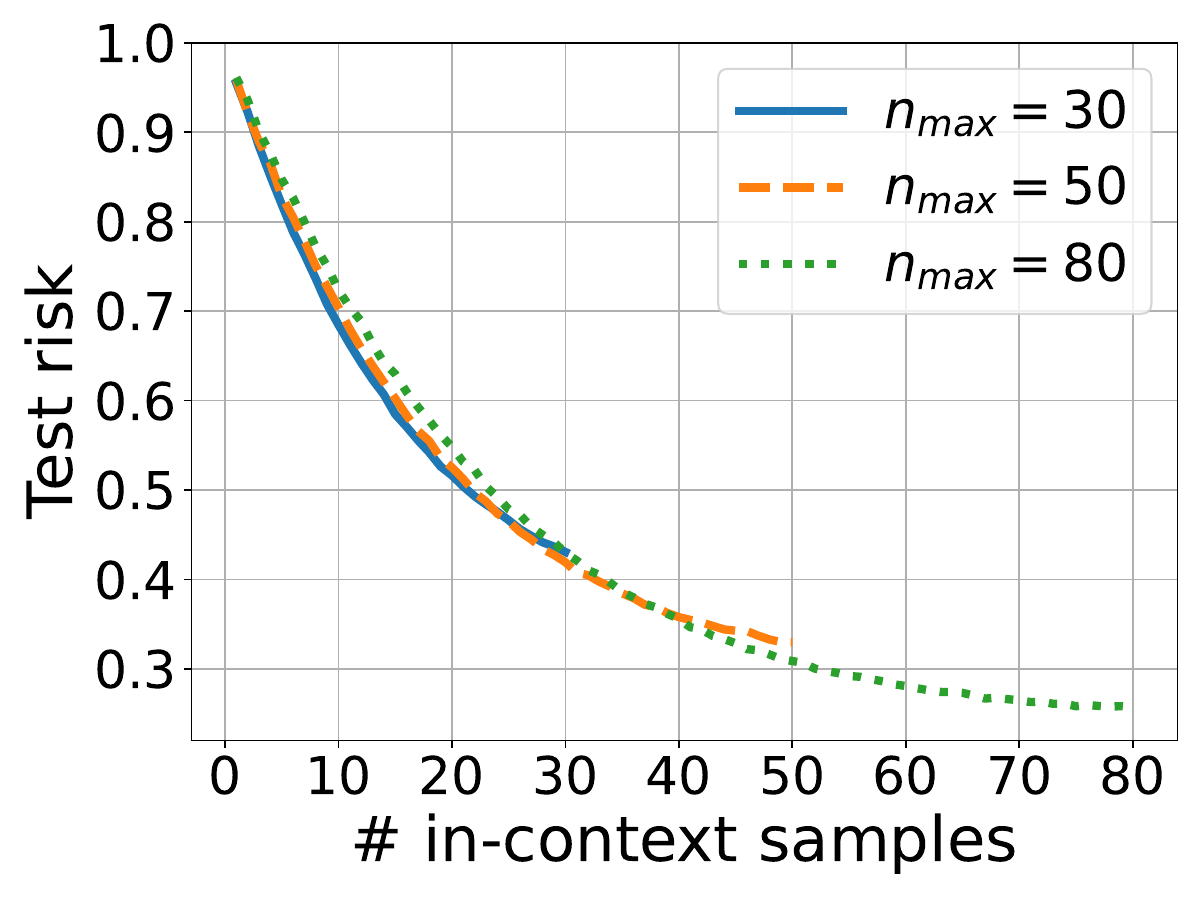}
\caption{H3}\label{fig:avg h3}
\end{subfigure}
\begin{subfigure}{0.244\linewidth}    
    \includegraphics[height=.75\columnwidth]{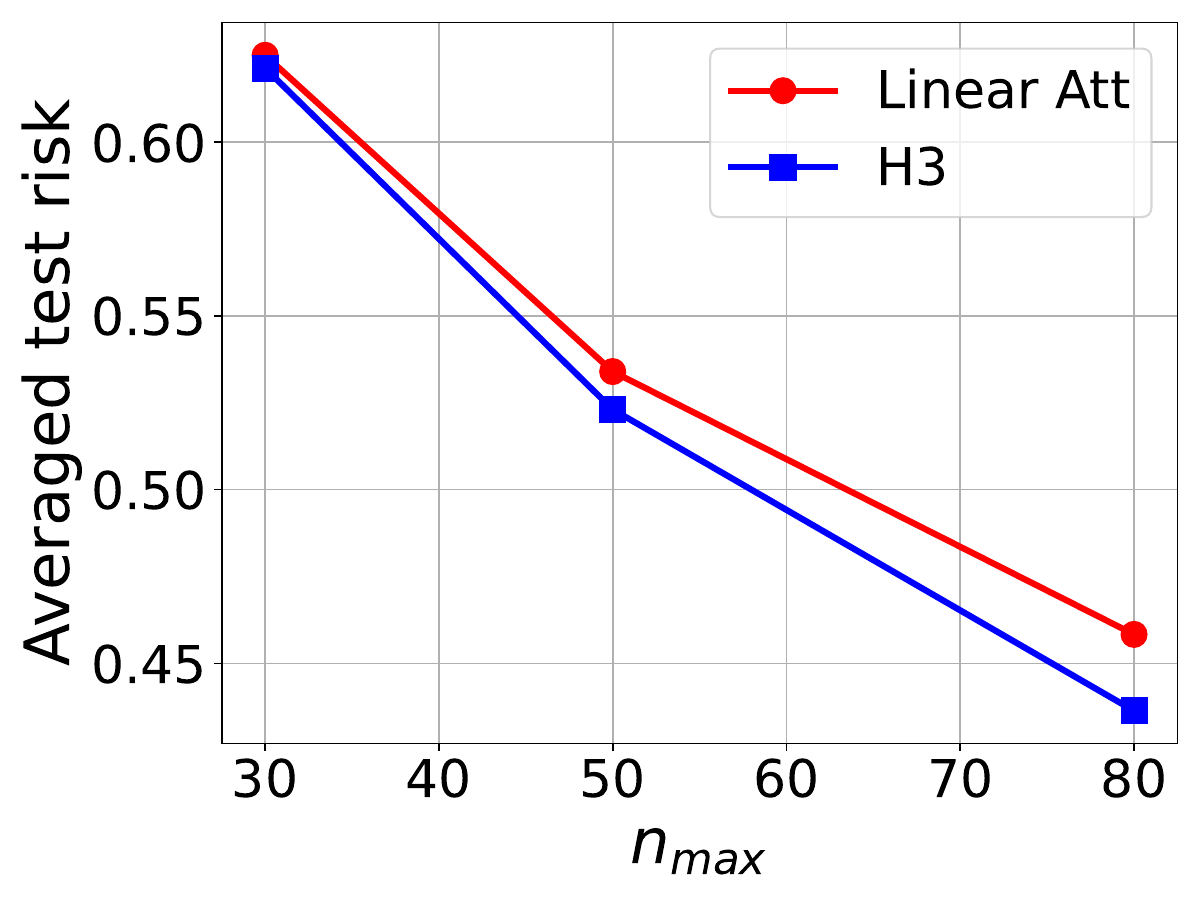}
\caption{Averaged risk}\label{fig:avg risk}
\end{subfigure}
\begin{subfigure}{0.244\linewidth}    
    \includegraphics[height=.75\columnwidth]{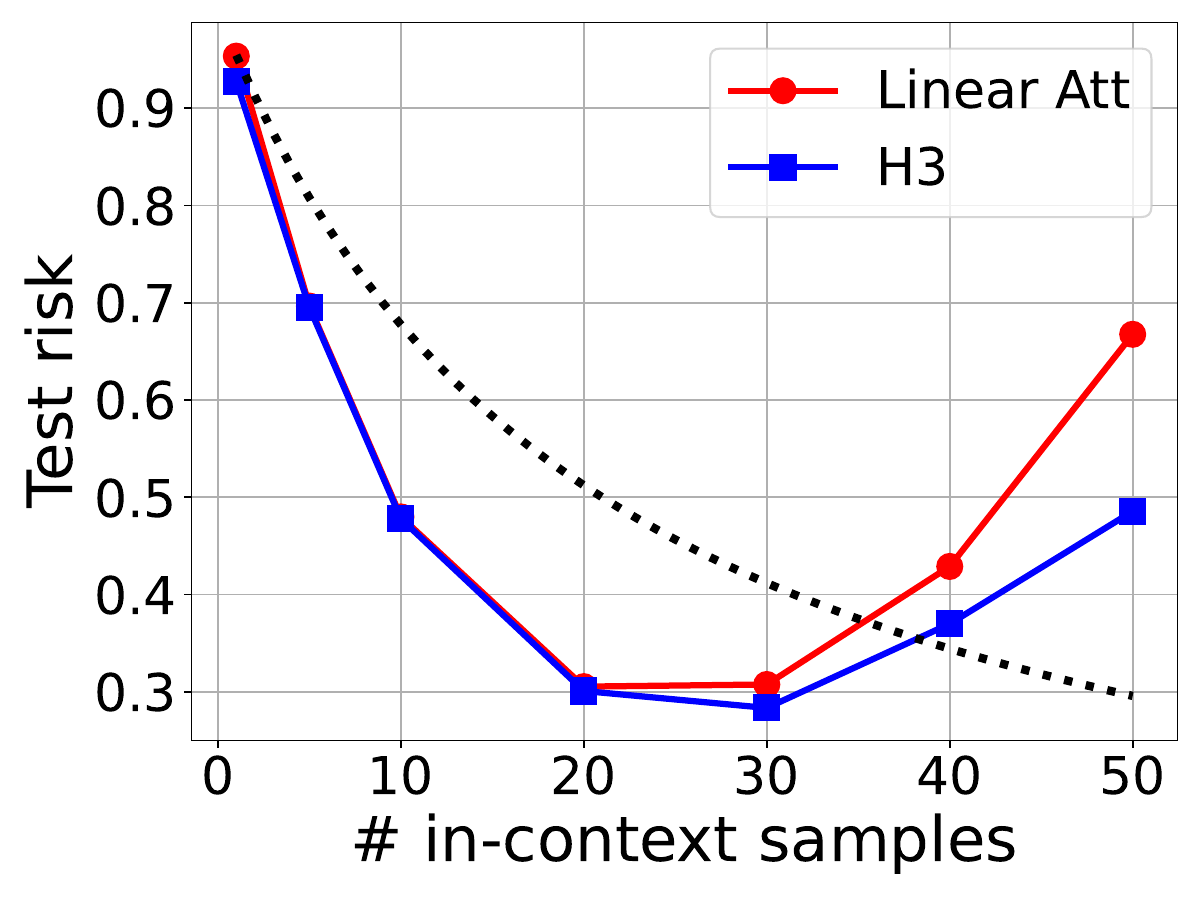}
\caption{Evolving $\bt$}\label{fig:evolve}
\end{subfigure}

\caption{Further comparison for linear attention and H3. In \textbf{(a)} and \textbf{(b)}, given maximum context lengths $n_{\max}$, we train linear attention and H3 models to minimize the average loss across all positions $n$ from $1$ to $n_{\max}$. Averaged test risks are presented in \textbf{(c)}. In \textbf{(d)}, the task vector $\bt$ evolves gradually over the context positions $i\leq n$ via $\bt_i=(i/n)\bt_1+(1-i/n)\bt_2$. In both scenarios, H3 outperforms linear attention benefiting from its additional convolutional filter (c.f.~$\fb$ in \eqref{ssm}). More experimental details are discussed in Section~\ref{sec exp}.}
\label{fig:add exp}
\end{figure}

\noindent$\bullet$ \textbf{H3 outperforms linear attention (Figure~\ref{fig:add exp}).} Until now, our analysis has established the equivalence between linear attention and H3 models in solving linear ICL problem. Furthermore, we also investigate settings where H3 could outperform linear attention due to its sample weighting ability. In Figs.~\ref{fig:avg att} and \ref{fig:avg h3}, instead of training separate models to fit the different context lengths, we train a single model with fixed max-length $n_{\max}$ and loss is evaluated as the average loss given samples from $1$ to $n_{\max}$. Such setting has been wildly studied in the previous ICL work \citep{garg2022can,akyrek2023what,li2023transformers}. We generate data according to \eqref{data ind} with $\bSi_{\x}=\bSi_{\bt}=\Iden_d$ and $\sigma=0$, and train 1-layer linear attention (Fig.~\ref{fig:avg att}) and H3 (Fig.~\ref{fig:avg h3}) models with different max-lengths $n_{\max}=30,50,80$. Comparison between Fig.~\ref{fig:avg att} and \ref{fig:avg h3} shows that 1-layer attention and H3 implement different algorithms in solving the averaged linear regression problem and H3 is more consistent in generalizing to longer context lengths. In Fig.~\ref{fig:avg risk}, we plot the averaged risks for each model and H3 outperforms linear attention. Furthermore, in Fig.~\ref{fig:evolve}, we focus on the setting where in-context examples are generated using evolving task vector $\bt$. Specifically, consider that each sequence corresponds to two individual task parameters $\bt_1\sim\Nc(0,\Iden_d)$ and $\bt_2\sim\Nc(0,\Iden_d)$. Then the $i$'th sample is generated via $\x_i\sim\Nc(0,\Iden_d)$ and $y_i=\bt_i^\top\x_i$ where $\bt_i=\lambda_i\bt_1+(1-\lambda_i)\bt_2$ and $\lambda_i=i/n$. The results are reported in Fig.~\ref{fig:evolve} which again shows that H3 achieves better performance compared to linear attention, as H3 may benefit from the additional convolutional filter (c.f. $\fb$ in \eqref{ssm}). Here, dotted curve represent the theoretical results under i.i.d. and noiseless setting, derived from Corollary~\ref{corol:iid}. 
\section{Related Work}

There is growing interest in understanding the mechanisms behind ICL~\citep{brown2020language,liu2023pre, rae2021scaling} in large language models (LLMs) due to its success in continuously enabling novel applications for LLMs~\citep{team2023gemini,gpt4_techreport,touvron2023_llama}. 
%
Towards this, \citet{xie2022an} explain ICL by language model's ability to perform implicit Bayesian inference where, under specific assumptions on the pre-training data distribution, the model infers a shared latent concept among the in-context examples and leverages the concept to make a prediction. \citet{muller2021transformers,hollmann2022tabpfn,muller2023pfns4bo} introduce prior-data fitted network (PFN) to approximate Bayesian inference on synthetic datasets and use it to perform downstream tasks such as tabular dataset classification. On the other hand, \citet{olsson2022context} posit induction heads as the key mechanism enabling ICL in Transformers. \citet{park2024can} study how various distributional properties of training data aid in the emergence of ICL in Transformers.

In the previous work, \citet{garg2022can} explored ICL ability of Transformers. In particular, they considered in-context prompts where each in-context example is labeled by a target function from a given function class, including linear models. A number of works have studied this and related settings to develop a theoretical understanding of ICL \citep{von2023uncovering,gatmirycan,collins2024context,lin2024dual,li2024dissecting,bai2024transformers,akyrek2023what,zhang2023trained,du2023can}.
\citet{akyrek2023what} focus on linear regression and provide a construction of Transformer weights that can enable a single step of GD based on in-context examples. They further show that Transformers trained on in-context prompts exhibit behaviors similar to the models recovered via explicit learning algorithm on the in-context examples in a prompt. Along the similar line, \citet{von2023transformers} provide a construction of weights in linear attention-only Transformers that can emulate GD steps on in-context examples for a linear regression task. Interestingly, they find similarity between their constructed networks and the networks resulting from training on in-context prompts corresponding to linear regression tasks. Similar to this line of work, \citet{dai2023gpt} argue that pre-trained language models act as meta-optimizer which utilize attention to apply meta-gradients to the original language model based on the in-context examples. Focusing on various NLP tasks, they further connect it to a specific form of explicit fine-tuning that performs gradient updates to the attention-related parameters. Inspired by the connection between linear attention and GD, they developed a novel attention mechanism that mirrors the behavior of GD with momentum. 
Beyond Transformers, existing work \citep{lee2023exploring,zucchet2023gated,grazzi2024mamba} demonstrate that other model architectures, such as SSM and RNNs, are also capable of in-context learning (ICL).  



Building on these primarily empirical studies, \citet{zhang2024trained, mahankali2024one, ahn2024transformers,duraisamy2024finite} focus on developing a theoretical understanding of Transformers trained to perform ICL. For single-layer linear attention model trained on in-context prompts for random linear regression tasks with isotropic Gaussian features and isotropic Gaussian weight vectors, \citet{mahankali2024one, ahn2024transformers} show that the resulting model implements a single step of GD on in-context examples in a test prompt, thereby corroborating the findings of \citep{von2023transformers}. They also show that the learned model implements a PGD step, when faced with anisotropic Gaussian features, with \citet{mahankali2024one} also considering anisotropic Gaussian weight vectors. \asrnote{Connect with our work.}\citet{ahn2024transformers} further study multi-layer model and show that the trained model can implement a generalization of GD++ algorithm, supporting an empirical observation in~\cite{von2023transformers}. On the other hand, \citet{mahankali2024one} extend their single-layer setup to consider suitable non-linear target functions, showing that learned Transformer again implements a single step of GD on lineare regression objective.
%
%
\asrnote{connect this to our setup and results ...Zhang et al. allow for anisotropic Gaussian feature vectors. – weight vector is still isotropic Gaussian…Akyurek et al. and Oswald et al. – both assume isotropic Gaussian features vectors and isotropic Gaussian parameter/weight vectors}
For a single-layer linear attention model, \citet{zhang2024trained} study the optimization dynamics of gradient flow while training such a model on in-context prompts for random linear regression tasks. Despite the non-convexity of the underlying problem, they show the convergence to the global minimum of the population objective. Similar to \citet{mahankali2024one, ahn2024transformers}, they show that the trained model implements a single step of GD and PGD for isotropic and anisotropic Gaussian features, respectively. In addition, they also characterize the test-time prediction error for the trained model while highlighting its dependence on train and test prompt lengths. 
Interestingly, \citet{zhang2024trained} further explore the effect of various distributional shifts, including the shift in task weight vector distributions between train and test time as well as the covariate shifts among train and test in-context prompts. Interestingly, they find that while linear-attention models are robust to most shifts, they exhibit brittleness to the covariate shifts. 

While our work shares similarities with this line of works, as discussed in our contributions in the introduction, we expand the theoretical understanding of ICL along multiple novel dimensions, which includes the first study of LoRA adaptation for ICL in the presence of a distributional shift. Furthermore, we strive to capture the effect of retrieval augmentation~\citep{lewis2020retrieval, nakano2021webgpt} on ICL through our analysis. Retrieval augmentation allows for selecting most relevant demonstration out of a large collection for a test instance, e.g., via a dense retrieval model~\citep{izacard2023atlas}, which can significantly outperform the typical ICL setup where fixed task-specific demonstrations are provided as in-context examples~\citep{wang2022training, basu2023statistical}. Through a careful modeling of retrieval augmentation via correlated design, we show that it indeed has a desirable amplification effect where the effective number in-context examples becomes larger with higher correlation which corresponds to preforming a successful retrieval of query-relevant demonstrations in a practical retrieval augmented setup.
%








Recently, state space models (SSMs)~\citep{gu2021combining, gu2021efficiently, fu2023hungry, gu2023mamba} have appeared as potential alternatives to Transformer architecture, with more efficient scaling to input sequence length. Recent studies demonstrate that such SSMs can also perform ICL for simple non-language tasks~\citep{park2024can, grazzi2024mamba} as well as complex NLP tasks~\citep{grazzi2024mamba}. That said, a rigorous theoretical understanding of ICL for SSMs akin to \cite{zhang2024trained, mahankali2024one, ahn2024transformers} is missing from the literature. In this work, we provide the first such theoretical treatment for ICL with SSMs. Focusing on H3 architecture~\citep{fu2023hungry}, we highlight its advantages over linear attention in specific ICL settings.












\section{Discussion}
In this work, we revisited the loss landscape of in-context learning with 1-layer sequence models. We have established a general connection between ICL and gradient methods that accounts for correlated data, non-attention architectures (specifically SSMs), and the impact of low-rank parameterization including LoRA adaptation. Our results elucidate two central findings: (1) The functions learned by different sequence model architectures exhibit a strong degree of \emph{universality} and (ii) \emph{Dataset and prompt design}, such as RAG, can substantially benefit ICL performance.

\noindent\textbf{Future directions and limitations.} The results of this work fall short of being a comprehensive theory for ICL in LLMs and can be augmented in multiple directions. First, while the exact equivalence between H3 and linear attention is remarkable, we should examine whether it extends to other SSMs. Secondly, while empirically predictive, our RAG and LoRA analyses are not precise and fully formal. Thirdly, it is desirable to develop a deeper understanding of multilayer architectures and connect to iterative GD methods as in \citep{ahn2024transformers,von2023transformers}. Finally, we have studied the population risk of ICL training whereas one can also explore the sample complexity of pretraining \citep{wu2023many,yuelu}. Moving beyond the theoretically tractable setup of this work, our simplified models are trained on in-context prompts from random initialization. Therefore, this theoretical study doesn't address more challenging in-context learning tasks, such as question answering, where both in-context demonstration and general knowledge from pretraining are required.  Future work in this area could also shed light on how certain contexts might elicit undesirable behaviors acquired by an LLM during pretraining, an aspect not covered in our current analysis. This work also studies a theoretical model for retrieval augmentation-based ICL. In a real-life retrieval augmentation-based ICL, one needs to account for the quality of the collection of the retrievable demonstrations and its (negative) impacts on the final predictions.

\subsection*{Acknowledgements}

This work was supported in part by the National Science Foundation grants CCF-2046816, CCF-2403075, the Office of Naval Research award N000142412289, an Adobe Data Science Research award, and a gift by Google Research.

\bibliography{ref}
\bibliographystyle{plainnat}

\appendix
\newpage


\addcontentsline{toc}{section}{Appendix} 
\part{Appendix} 
\parttoc 
\section{Equivalence among Gradient Descent, Attention, and State-Space Models}
In this section, we present the proofs related to Section~\ref{sec:setup}. 
Recap that given data
\begin{align*}
&\X=[\x_1~\cdots~\x_n]^\top\in\R^{n\times d},\\
&\bxi=[\xi_1~\cdots~\xi_n]^\top\in\R^n,\\
&\y=[y_1~\cdots~y_n]^\top=\X\bt+\bxi\in\R^n,\\
&
\Z_0=[\z_1~\dots~\z_n~\zerbb_{d+1}]^\top=\begin{bmatrix}
    \x_1&\dots&\x_n&\zerbb_d\\
    y_1&\dots&y_n&0
\end{bmatrix}^\top \in \R^{(n+1) \times (d+1)},
\end{align*}
and corresponding prediction functions
\begin{subequations}\label{pred all app}
\begin{align}
&g_{\gd}(\Z)=\x^\top\W\X^\top\y,\label{pred gd app}\\
&g_{\wpgd}(\Z)=\x^\top\W\X^\top(\bom\odot\y),\label{pred wpgd app}\\
&g_{\att}(\Z)=(\z^\top\WQ\WK^\top\Zm^\top)\Zm\WV\vb,\label{pred att app}\\
&g_{\ssm}(\Z)=\left((\z^\top\WQ)^\top\odot((\Zm\WK\odot\Zm\WV) \ast\fb)_{n+1}\right)\vb,\label{pred ssm app}
\end{align}
\end{subequations}
we have objectives 
\begin{subequations}\label{obj all app}
\begin{align}
\min_{\W}\Lc_{\gd}(\Wc)\quad&\text{where}\quad\Lc_{\gd}(\Wc)=
\E\left[(y-g_{\gd}(\Z))^2\right]\label{obj gd app},\\
\min_{\W,\bom}\Lc_{\wpgd}(\Wc)\quad&\text{where}\quad\Lc_{\wpgd}(\Wc)=
\E\left[(y-g_{\wpgd}(\Z))^2\right]\label{obj wpgd app},\\
\min_{\WK,\WQ,\WV,\vb}\Lc_{\att}(\Wc)\quad&\text{where}\quad\Lc_{\att}(\Wc)=
\E\left[(y-g_{\att}(\Z))^2\right]\label{obj att app},\\
\min_{\WK,\WQ,\WV,\vb,\fb}\Lc_{\ssm}(\Wc)\quad&\text{where}\quad\Lc_{\ssm}(\Wc)=
\E\left[(y-g_{\ssm}(\Z))^2\right].\label{obj ssm app}
\end{align}
\end{subequations}
Here, the expectation is over the randomness in $(\x_i,\xi_i)_{i=1}^{n}$ and $\bt$, and the search space for $\W$ is $\R^{d\times d}$, for $\bom$ is $\R^{n}$, for $\WK,\WQ,\WV$ is $\R^{(d+1)\times(d+1)}$, for $\vb$ is $\R^{d+1}$, and for $\fb$ is $\R^{n+1}$. 

\subsection{Proof of Proposition~\ref{lemma:eqv}}\label{app:proof eqv}
Consider the problem setting as discussed in Section~\ref{sec:setup}, 
Proposition~\ref{lemma:eqv} can be proven by the following two lemmas.
\begin{lemma}\label{lemma gd=att}
Suppose Assumptions~\ref{assum:odd x beta} and \ref{assume:noise} hold. Then, given the objectives \eqref{obj gd app} and \eqref{obj att app}, we have
\[
\min_{\WQ,\WK,\WV,\vb}\Lc_{\att}(\Wc)=\min_{\W}\Lc_{\gd}(\Wc).
\]
\end{lemma}

\begin{proof}
Recap the linear attention estimator from \eqref{pred att app} and denote 
\[
\WQ\WK^\top=\begin{bmatrix}
    \Wb&\w_1\\
    \w_2^\top&w
\end{bmatrix}\qquad\text{and}\qquad\WV\vb=\begin{bmatrix}
    \vb_1\\
    v
\end{bmatrix},
\]
where $\Wb\in\R^{d\times d}$, $\w_1,\w_2,\vb_1\in\R^d$, and $w,v\in\R$. 
Then we have 
\begin{align}
g_{\att}(\Z)&=(\z^\top\WQ\WK^\top\Z_0^\top)\Z_0\WV\vb\nn\\
&=[\x^\top~0]\begin{bmatrix}
    \Wb&\w_1\\
    \w_2^\top&w
\end{bmatrix}\begin{bmatrix}
    \X^\top&\zerbb_d\\
    \y^\top&0
\end{bmatrix}\begin{bmatrix}
    \X&\y\\
    \zerbb_d^\top&0
\end{bmatrix}\begin{bmatrix}
    \vb_1\\
    v
\end{bmatrix}\nn\\
&=(\x^\top\Wb\X^\top+\x^\top\w_1\y^\top)(\X\vb_1+\y v)\nn\\
&=\x^\top(v\Wb)\X^\top\y+\x^\top\w_1\y^\top\X\vb_1+\x^\top\left(\Wb\X^\top\X\vb_1+v\tn{\y}^2\w_1\right)\nn\\
&=\x^\top(v\Wb+\w_1\vb_1^\top)\X^\top\y+\x^\top\left(\Wb\X^\top\X\vb_1+v\tn{\y}^2\w_1\right)\nn\\
&=\underset{\tilde{g}_{\att}(\Z)}{\underbrace{\x^\top\Wt\X^\top\y}}+\underset{\eps}{\underbrace{\x^\top\left(\Wb\X^\top\X\vb_1+v\tn{\y}^2\w_1\right)}},\label{gd=att eq 1}
\end{align}
where $\Wt:=v\Wb+\w_1\vb_1^\top$. 

We first show that for any given parameters $\WK,\WQ,\WV,\vb$,  
\begin{align}
\E\left[(g_{\att}(\Z)-y)^2\right]\geq\E\left[(\tilde g_{\att}(\Z)-y)^2\right].\label{lower bound}
\end{align}
To this goal, we have
\begin{align}
    \E\left[(g_{\att}(\Z)-y)^2\right]-\E\left[(\tilde g_{\att}(\Z)-y)^2\right]
    &=\E\left[(\tilde g_{\att}(\Z)+\eps-y)^2\right]-\E\left[(\tilde g_{\att}(\Z)-y)^2\right]\nn\\
    &=\E[\eps^2]+2\E[(\tilde g_{\att}(\Z)-y)\eps]\label{risk gap}
\end{align}
where we have decomposition
\begin{align*}
    (\tilde g_{\att}(\Z)-y)\eps&=(\x^\top\Wt\X^\top\y-y)\x^\top\left(\Wb\X^\top\X\vb_1+v\tn{\y}^2\w_1\right)\\
    &=\y^\top\X\Wt^\top\x\x^\top\left(\Wb\X^\top\X\vb_1+v\tn{\y}^2\w_1\right)-y\x^\top\left(\Wb\X^\top\X\vb_1+v\tn{\y}^2\w_1\right)\\
    &=\underset{(a)}{\underbrace{\y^\top\X\Wt^\top\x\x^\top\Wb\X^\top\X\vb_1}}+\underset{(b)}{\underbrace{v\tn{\y}^2\y^\top\X\Wt^\top\x\x^\top\w_1}}-\underset{(c)}{\underbrace{y\x^\top\Wb\X^\top\X\vb_1}}-\underset{(d)}{\underbrace{ vy\tn{\y}^2\x^\top\w_1}}.
\end{align*}
In the following, we consider the expectations of $(a),(b),(c),(d)$ sequentially, which return zeros under Assumptions~\ref{assum:odd x beta} and \ref{assume:noise}. Note that since Assumption~\ref{assum:odd x beta} holds, expectation of any odd \emph{order} of monomial of the entries of $\X,\x,\bt$ returns zero, i.e., 
order of $\x^\top\bt\x$ is 3 and therefore  $\E[\x^\top\bt\x]=\zerbb_d$. 
\begin{align*}
    (a):\quad&\E\left[\y^\top\X\Wt^\top\x\x^\top\Wb\X^\top\X\vb_1\right]\hspace*{300pt}\\
    &=\E\left[(\X\bt+\bxi)^\top\X\Wt^\top\x\x^\top\Wb\X^\top\X\vb_1\right]\\
    &=\E\left[\bt^\top\X^\top\X\Wt^\top\x\x^\top\Wb\X^\top\X\vb_1\right]+\E\left[\bxi^\top\X\Wt^\top\x\x^\top\Wb\X^\top\X\vb_1\right]\\
    &=0.
\end{align*}
\begin{align*}
    (b):\quad&\E\left[v\tn{\y}^2\y^\top\X\Wt^\top\x\x^\top\w_1\right]\hspace*{300pt}\\
    &=\E\left[v(\X\bt+\bxi)^\top(\X\bt+\bxi)(\X\bt+\bxi)^\top\X\Wt^\top\x\x^\top\w_1\right]\\
    &=\E\left[v\tn{\bxi}^2\bxi^\top\X\Wt^\top\x\x^\top\w_1\right]\\
    &=0.
\end{align*}
\begin{align*}
    (c):\quad&\E\left[y\x^\top\Wb\X^\top\X\vb_1\right]\hspace*{300pt}\\
    &=\E\left[(\x^\top\bt+\xi)\x^\top\Wb\X^\top\X\vb_1\right]\\
    &=\E\left[\bt^\top\x\x^\top\Wb\X^\top\X\vb_1\right]+\E\left[\xi\x^\top\Wb\X^\top\X\vb_1\right]\\
    &=0.
\end{align*}
\begin{align*}
    (d):\quad&\E\left[vy\tn{\y}^2\x^\top\w_1\right]\hspace*{300pt}\\
    &=v\E\left[(\bt^\top\x+\xi)(\X\bt+\bxi)^\top(\X\bt+\bxi)\x^\top\w_1\right]\\
    &=v\E\left[\xi\tn{\bxi}^2\x^\top\w_1\right]\\
    &=0.
\end{align*}
Combining the results with \eqref{risk gap} returns that 
\begin{align}
\E\left[(g_{\att}(\Z)-y)^2\right]-\E\left[(\tilde g_{\att}(\Z)-y)^2\right]=\E[\eps^2]\geq0\label{gd=att eq 2}
\end{align}
which completes the proof of \eqref{lower bound}. Therefore, we obtain
\[
\min_{\WQ,\WK,\WV,\vb}\E\left[(g_{\att}(\Z)-y)^2\right]\geq\min_{\tilde\W}\E\left[(\tilde g_{\att}(\Z)-y)^2\right]=\min_{\W}\E\left[(g_{\gd}(\Z)-y)^2\right].
\]
We conclude the proof of this lemma by showing that for any $\W\in\R^{d\times d}$ in $g_{\gd}$, there exist $\WK,\WQ,\WV,\vb$ such that $g_\att(\Z)=g_{\gd}(\Z)$. Let 
\[
\WK=\WV=\Iden_{d+1},\qquad\WQ=\begin{bmatrix}
    \W&\zerbb_d\\
    \zerbb_d^\top&0
\end{bmatrix},\quad\text{and}\quad\vb=\begin{bmatrix}
    \zerbb_d\\1
\end{bmatrix}.
\]
Then we obtain
\begin{align}
g_{\att}(\Z)=\x^\top\W\X^\top\y=g_{\gd}(\Z), \label{gd=att eq 3}
\end{align}
which completes the proof.
\end{proof}

\begin{lemma}Suppose Assumptions~\ref{assum:odd x beta} and \ref{assume:noise} hold. Then, given the objectives in \eqref{obj all app}, we have
\begin{align}
\min_{\WQ,\WK,\WV,\vb,\fb}\Lc_{\ssm}(\Wc)=\min_{\W,\bom}\Lc_{\wpgd}(\Wc).\label{lemma ssm=wpgd eq1}
\end{align}
{Additionally, if the examples $(\x_i,y_i)_{i=1}^n$ follow the same distribution and are conditionally independent given $\x$ and $\bt$, then SSM/H3 can achieve the optimal loss using the all-ones filter and
\begin{align}
\min_{\W,\bom}\Lc_{\wpgd}(\Wc)=\min_{\W}\Lc_{\gd}(\Wc).\label{lemma ssm=wpgd eq2}
\end{align}}
\end{lemma}
\begin{proof}
Recap the SSM estimator from \eqref{pred ssm app} and let 
\begin{align*}
    &\WQ=\begin{bmatrix}
        \w_{q1}&
        \w_{q2}&
        \cdots&
        \w_{q,{d+1}}
    \end{bmatrix},\\
    &\WK=\begin{bmatrix}
        \w_{k1}&
        \w_{k2}&
        \cdots&
        \w_{k,{d+1}}
    \end{bmatrix},\\
    &\WV=\begin{bmatrix}
        \w_{v1}&
        \w_{v2}&
        \cdots&
        \w_{v,{d+1}}
    \end{bmatrix},
\end{align*}
where $\w_{qj},\w_{kj},\w_{vj}\in\R^{d+1}$ for $j\leq d+1$, and let
\[
\vb=\begin{bmatrix}
    v_1\\
    v_2\\
    \cdots\\
    v_{d+1}
\end{bmatrix},\quad\text{and}\quad\fb=\begin{bmatrix}
    f_{0}\\
    f_{1}\\
    \cdots\\
    f_{n}
\end{bmatrix}.
\]
Then we have
\begin{align*}
    g_\ssm(\Z)&=\left((\z^\top\WQ)^\top\odot((\Zm\WK\odot\Zm\WV) \ast\fb)_{n+1}\right)\vb\\
    &=\sum_{i=1}^nf_{n+1-i}\cdot\vb^\top\left(\begin{bmatrix}
        \w_{q1}^\top\z\\
        \cdots\\
        \w_{q,{d+1}}^\top\z
    \end{bmatrix}\odot\begin{bmatrix}
        \w_{k1}^\top\z_i\w_{v1}^\top\z_i\\
        \cdots\\
        \w_{k,{d+1}}^\top\z_i\w_{v,{d+1}}^\top\z_i
    \end{bmatrix}\right)\\
    &=\sum_{i=1}^nf_{n+1-i}\cdot\vb^\top\begin{bmatrix}
        \w_{q1}^\top\z\w_{k1}^\top\z_i\w_{v1}^\top\z_i\\
        \cdots\\
        \w_{q,{d+1}}^\top\z\w_{k,{d+1}}^\top\z_i\w_{v,{d+1}}^\top\z_i
    \end{bmatrix}.
\end{align*}
Next for all $j\leq d+1$, let 
\[
\w_{qj}=\begin{bmatrix}
    \wb_{qj}\\
    w_{qj}
\end{bmatrix},\quad\w_{kj}=\begin{bmatrix}
    \wb_{kj}\\
    w_{kj}
\end{bmatrix},\quad\w_{vj}=\begin{bmatrix}
    \wb_{vj}\\
    w_{vj}
\end{bmatrix}
\]
where $\wb_{qj},\wb_{kj},\wb_{vj}\in\R^d$ and $w_{qj},w_{kj},w_{vj}\in\R$. 
Then we have
\begin{align*}
    \w_{qj}^\top\z\w_{kj}^\top\z_i\w_{vj}^\top\z_i&=\left(\wb_{qj}^\top\x\right)\left(\wb_{kj}^\top\x_i+w_{kj}y_i\right)\left(\wb_{vj}^\top\x_i+w_{vj}y_i\right)\\
    &=\x^\top\wb_{qj}\left(w_{vj}\wb_{kj}^\top+w_{kj}\wb_{vj}^\top\right)\x_iy_i+\left(\wb_{qj}^\top\x\right)\left(\wb_{kj}^\top\x_i\right)\left(\wb_{vj}^\top\x_i\right)+\left(w_{kj}w_{vj}\wb_{qj}^\top\x y_i^2\right)\\
    &=\x^\top\W'_j\x_i y_i+\delta_j(\x,\x_i,\x_i)+{\w_j'}^\top \x y_i^2
\end{align*}
where 
\begin{align*}
&\W'_j:=\wb_{qj}\left(w_{vj}\wb_{kj}^\top+w_{kj}\wb_{vj}^\top\right)\in\R^{d\times d},\\
&\w_j':=w_{kj}w_{vj}\wb_{qj}\in\R^d,\\
&\delta_j(\x,\x_i,\x_i):=\left(\wb_{qj}^\top\x\right)\left(\wb_{kj}^\top\x_i\right)\left(\wb_{vj}^\top\x_i\right)\in\R.
\end{align*}
Then
\begin{align*}
    g_{\ssm}(\Z)&=\sum_{i=1}^nf_{n+1-i}\cdot\sum_{j=1}^{d+1}v_j\left(\x^\top\W'_j\x_i y_i+\delta_j(\x,\x_i,\x_i)+{\w_j'}^\top \x y_i^2\right)\\
    &=\x^\top\left(\sum_{j=1}^{d+1}v_j\W'_j\right)\X(\y\odot\tilde\fb)+\sum_{i=1}^nf_{n+1-i}\cdot\sum_{j=1}^{d+1}v_j\cdot\delta_j(\x,\x_i,\x_i)+\left(\sum_{j=1}^{d+1}v_j{\w_j'}^\top\right)\x \y^\top(\y\odot\tilde\fb)\\
    &=\underset{\tilde g_{\ssm}(\Z)}{\underbrace{\x^\top\tilde\W\X\tilde\y}}+\underset{\eps_1}{\underbrace{\tilde\delta(\x,\X,\X)}}+\underset{\eps_2}{\underbrace{\tilde\w^\top\x \y^\top\tilde\y}}.
\end{align*}
where
\begin{align*}
&\tilde\fb:=[f_n~\cdots~f_1]^\top\in\R^n,\\
&\tilde\y:=\y\odot\tilde\fb\in\R^n,\\
&\tilde\W:=\sum_{j=1}^{d+1}v_j\W'_j\in\R^{d\times d},\\
&\tilde\w:=\sum_{j=1}^{d+1}v_j{\w_j'}\in\R^d,\\
&\tilde\delta(\x,\X,\X):=\sum_{i=1}^nf_{n+1-i}\cdot\sum_{j=1}^{d+1}v_j\cdot\delta_j(\x,\x_i,\x_i)\in\R.
\end{align*}
Next we will show that for any $\WK,\WQ,\WV,\vb$, 
\[
\E\left[(g_{\ssm}(\Z)-y)^2\right]\geq\E\left[(\tilde g_{\ssm}(\Z)-y)^2\right].
\]
To start with, we obtain
\begin{align}
    \E\left[(g_\ssm(\Z)-y)^2\right]&=\E\left[(\tilde g_{\ssm}(\Z)+\eps_1+\eps_2-y)^2\right]\nn\\
    &=\E\left[(\tilde g_{\ssm}(\Z)-y)^2\right]+\E\left[(\eps_1+\eps_2)^2\right]+2\E\left[(\tilde g_{\ssm}(\Z)-y)(\eps_1+\eps_2)\right]\label{risk gap ssm}
\end{align}
where there is decomposition
\begin{align*}
    (\tilde g_{\ssm}(\Z)-y)(\eps_1+\eps_2)=\underset{(a)}{\underbrace{\tilde\delta(\x,\X,\X)\cdot\x^\top\tilde\W\X\tilde\y}}-\underset{(b)}{\underbrace{\tilde\delta(\x,\X,\X)y}}+\underset{(c)}{\underbrace{\tilde\w^\top\x \y^\top\tilde\y\cdot\x^\top\tilde\W\X\tilde\y}}-\underset{(d)}{\underbrace{y\cdot\tilde\w^\top\x \y^\top\tilde\y}}.
\end{align*}
In the following, similar to the proof of Lemma~\ref{lemma gd=att}, we consider the expectations of $(a),(b),(c),(d)$ sequentially, which return zeros under Assumptions~\ref{assum:odd x beta} and \ref{assume:noise}. Note that $\delta_j(\x,\x_i,\x_i)$'s and $\tilde\delta(\x,\X,\X)$ are summation of monomials of entries of $(\x,\X,\bt)$ with order 3, and entries of $\y$ and $y$ are summation of monomials of entries of $(\x,\X,\bt)$ with even orders: e.g., $y=\x^\top\bt+\xi$ where $\xi$ is of oder 0 and  $\x^\top\bt$ is of order 2.

\begin{align*}
    (a):\quad&\E\left[\tilde\delta(\x,\X,\X)\cdot\x^\top\tilde\W\X\tilde\y\right]\hspace*{300pt}\\
    &=\E\left[\tilde\delta(\x,\X,\X)\cdot\x^\top\tilde\W\X(\X\bt\odot\tilde\fb)\right]+\E\left[\tilde\delta(\x,\X,\X)\cdot\x^\top\tilde\W\X(\bxi\odot\tilde\fb)\right]\\
    &=\E\left[\tilde\delta(\x,\X,\X)\cdot\x^\top\tilde\W\X\right]\E\left[\bxi\odot\tilde\fb\right]\\
    &=0.
\end{align*}
\begin{align*}
    (b):\quad&\E\left[\tilde\delta(\x,\X,\X)y\right]\hspace*{400pt}\\
    &=\E\left[\tilde\delta(\x,\X,\X)(\x^\top\bt+\xi)\right]\\
    &=\E\left[\tilde\delta(\x,\X,\X)\x^\top\bt\right]+\E\left[\tilde\delta(\x,\X,\X)\xi\right]\\
    &=0.
\end{align*}
\begin{align*}
    (c):\quad&\E\left[\tilde\w^\top\x \y^\top\tilde\y\cdot\x^\top\tilde\W\X\tilde\y\right]\hspace*{400pt}\\
    &=\E\left[\tilde\w^\top\x (\X\bt+\bxi)^\top(\X\bt\odot\tilde\fb+\bxi\odot\tilde\fb)\cdot\x^\top\tilde\W\X(\X\bt\odot\tilde\fb+\bxi\odot\tilde\fb)\right]\\
    &=0.
\end{align*}
\begin{align*}
    (d):\quad&\E\left[y\cdot\tilde\w^\top\x \y^\top\tilde\y\right]\hspace*{400pt}\\
    &=\E\left[(\x^\top\bt+\xi)\cdot\tilde\w^\top\x (\X\bt+\bxi)^\top(\X\bt\odot\tilde\fb+\bxi\odot\tilde\fb)\right]\\
    &=0.
\end{align*}
Combining the results with \eqref{risk gap ssm} results that
\begin{align*}
    \E\left[(g_\ssm(\Z)-y)^2\right]-\E\left[(\tilde g_\ssm(\Z)-y)^2\right]=\E\left[(\eps_1+\eps_2)^2\right]\geq0.
\end{align*}
Therefore we obtain,
\[
\min_{\WQ,\WK,\WV,\vb,\fb}\E\left[(g_{\ssm}(\Z)-y)^2\right]\geq\min_{\tilde\W,\tilde\fb}\E\left[(\tilde g_{\ssm}(\Z)-y)^2\right]=\min_{\W,\bom}\E\left[( g_{\wpgd}(\Z)-y)^2\right].
\]
Next we show that for any choices of $\W$ and $\bom$ in $g_{\wpgd}$, there are $\W_{q,k,v},\vb,\fb$ such that $g_{\ssm}\equiv g_{\wpgd}$. To this end, given $\bom=[\omega_1~\dots~\omega_n]^\top$, let 
\[
\WQ=\Iden_{d+1},\quad\WK=\begin{bmatrix}
    \W^\top&\zerbb_d\\
    \zerbb_d^\top&0
\end{bmatrix},\quad\WV=\begin{bmatrix}
    \zerbb_{d\times d}&\zerbb_d\\
    \onebb_d^\top&0
\end{bmatrix},\quad\vb=\begin{bmatrix}
    \onebb_d\\
    0
\end{bmatrix}\quad\text{and}\quad\fb=\begin{bmatrix}
    0\\
    \omega_n\\
    \cdots\\
    \omega_1
\end{bmatrix}.
\]
Then we get
\begin{align*}
((\Z_0\WK\odot\Z_0\WV)\ast\fb)_{n+1}&=\left(\left(\begin{bmatrix}
    \X\W^\top&\zerbb_n\\
    \zerbb_d&0
\end{bmatrix}\odot\begin{bmatrix}
    \y\onebb_d^\top&\zerbb_n\\
    \zerbb_d&0
\end{bmatrix}\right)\ast\fb\right)_{n+1}\\
&=\begin{bmatrix}\sum_{i=1}^n\omega_i\cdot y_i\W\x_i\\
0\end{bmatrix}\\
&=\begin{bmatrix}\W\X^\top(\y\odot\bom)\\0\end{bmatrix},
\end{align*}
and therefore
\[
g_\ssm(\Z)=\x^\top\W\X^\top(\y\odot\bom)=g_{\wpgd}(\Z),
\]
which completes the proof of \eqref{lemma ssm=wpgd eq1}.

Next, to show \eqref{lemma ssm=wpgd eq2}, for any $\W\in\R^{d\times d}$, let $\Lc(\bom)=\E\left[\left(\x^\top\W\X^\top(\y\odot\bom)-y\right)^2\right]$.
Then we have
\begin{align*}
\frac{\partial\Lc(\bom)}{\partial \omega_i}&=\E\left[2\left(\x^\top\W \sum_{j=1}^n\omega_jy_j\x_j-y\right)\left(\x^\top\W y_i\x_i\right)\right]\\
&=2\sum_{j=1}^n\omega_j\E\left[(\x^\top\W y_j\x_j)(\x^\top\W y_i\x_i)\right]-2\E\left[y\x^\top\W y_i\x_i\right].
\end{align*}
Here since $(\x_i,y_i)_{i=1}^n$ follow the same distribution and are conditionally independent given $\x$ and $\bt$, for any $i\neq j\neq j'$, $\E\left[(\x^\top\W y_i\x_i)^2\right]=\E\left[(\x^\top\W y_{j}\x_{j})^2\right]$ and $\E\left[(\x^\top\W y_j\x_j)(\x^\top\W y_i\x_i)\right]=\E\left[(\x^\top\W y_{j'}\x_{j'})(\x^\top\W y_i\x_i)\right]$. Then let
\[
\E\left[(\x^\top\W y_j\x_j)(\x^\top\W y_i\x_i)\right]=\begin{cases}
    c_1,& i\neq j\\
    c_2,& i=j
\end{cases}\quad\text{and}\quad\E\left[y\x^\top\W y_i\x_i\right]=c_3,
\]
where $(c_1,c_2,c_3):=(c_1(\W),c_2(\W),c_3(\W))$. 
We get
\[
\frac{\partial\Lc(\bom)}{\partial \omega_i}
=2c_1\bom^\top\onebb_n+2(c_2-c_1)\omega_i-2c_3.
\]
If $c_2-c_1=0$, then $\frac{\partial\Lc(\bom)}{\partial \omega_i}\equiv2c_1\bom^\top\onebb_n-2c_3$ for all $i\leq n$ and any $\bom\in\R^n$ achieves the same performance.

If $c_2-c_1\neq0$, setting $\frac{\partial\Lc(\bom)}{\partial \omega_i}=0$ returns 
\[
\omega_i=\frac{c_3-c_1\sum_{j=1}^n\omega_j}{c_2-c_1}:=C\quad\text{for all }i\leq n.
\]
Therefore the optimal loss is achieved via setting $\bom=C\onebb_n$. Without loss of generality, we can update $\W\to C\W$. Then $\bom=\onebb_n$, and we obtain 
\[
\min_{\W,\bom}\E\left[\left(\x^\top\W\X^\top(\y\odot\bom)-y\right)^2\right]=\min_{\W}\E\left[(\x^\top\W\X^\top\y-y)^2\right]
\]
which completes the proof of \eqref{lemma ssm=wpgd eq2}.
\end{proof}

\subsection{Proof of Lemma~\ref{lemma cvx 1}}

\begin{proof}
    Recap the loss $\Lc_{\gd}(\Wc)$ in \eqref{obj gd app} and prediction $g_{\gd}(\Z)$ in \eqref{pred gd app}, we have
    \begin{align}
        \Lc_{\gd}(\Wc)&=\E[(y-g_{\gd}(\Z))^2]\nn\\
        &=\E\left[\left(\x^\top\bt+\xi-\x^\top\W\X^\top(\X\bt+\bxi)\right)^2\right]\nn\\
    &=\E\left[(\x^\top\bt-\x^\top\W\X^\top\X\bt)^2+2(\x^\top\bt-\x^\top\W\X^\top\X\bt)(\xi-\x^\top\W\X^\top\bxi)+(\xi-\x^\top\W\X^\top\bxi)^2\right]\nn\\
    &=\E\left[(\x^\top\bt-\x^\top\W\X^\top\X\bt)^2+(\xi-\x^\top\W\X^\top\bxi)^2\right]+2\E[(\x^\top\bt-\x^\top\W\X^\top\X\bt)(\xi-\x^\top\W\X^\top\bxi)]\nn\\
    &=\E\left[(\x^\top\bt-\x^\top\W\X^\top\X\bt)^2+(\xi-\x^\top\W\X^\top\bxi)^2\right]\label{cvx eq 1}\\
    &=\underset{f_1(\W)}{\underbrace{\E\left[(\x^\top\W\X^\top\X\bt)^2+(\x^\top\W\X^\top\bxi)^2\right]}}\underset{f_2(\W)}{\underbrace{-2\E[\bt^\top\x\x^\top\W\X^\top\X\bt+\xi\x^\top\W\X^\top\bxi]}}+\underset{\text{constant}}{\underbrace{\E[(\x^\top\bt)^2+\xi^2]}}\nn
    \end{align}
    where \eqref{cvx eq 1} follows Assumption \ref{assume:noise}. Since $f_2(\W)$ is convex, $\Lc_{\gd}(\Wc)$ is strongly-convex if and only if $f_1(\W)$ is strongly-convex, which completes the proof of strong convexity.

    Next, \eqref{gd=att eq 2} and \eqref{gd=att eq 3} in the proof of Lemma~\ref{lemma gd=att} demonstrate that the optimal loss is achievable and is achieved at $\eps=0$. Subsequently, \eqref{gd=att eq 1} indicates that $g^\st_{\att}$ has the same form as $g^\st_{\gd}$. Under the strong convexity assumption, $g^\st_{\gd}$ is unique, which leads to the conclusion that $g^\st_{\gd} = g^\st_{\att}$.
\end{proof}

\subsection{Proof of Lemma~\ref{lemma cvx 2}}
\begin{proof} 
According to Lemma~\ref{lemma cvx 1}, $\Lc_{\gd}(\Wc)$ is strongly-convex as long as either $\E[(\x^\top\W\X^\top\X\bt)^2]$ or $\E[(\x^\top\W\X^\top\bxi)^2]$ is strongly-convex. Therefore, in this lemma, the two claims correspond to the strong convexity of $\E[(\x^\top\W\X^\top\bxi)^2]$ and $\E[(\x^\top\W\X^\top\X\bt)^2]$ terms, respectively. 

Suppose the decomposition claim holds. Without losing generality, we may assume $(\x_1,\bt_1,\X_1)$ are zero-mean because we can allocate the mean component to $(\x_2,\bt_2,\X_2)$ without changing the covariance. 

\noindent $\bullet$ \textbf{Claim 1:} Let $\bSb_{\x}=\E[\x_1\x_1^\top]$, $\bSb_{\bt}=\E[\bt_1\bt_1^\top]$, and $\bSb_{\X}=\E[\X_1^\top\X_1]$. If the first claim holds, using independence, observe that we can write
\begin{align}
\E[(\x^\top\W\X^\top\bxi)^2]&=\E[(\x_1^\top\W\X_1^\top\bxi)^2]+\E[(\x_1^\top\W\X_2^\top\bxi)^2]+\E[(\x_2^\top\W\X_1^\top\bxi)^2]+\E[(\x_2^\top\W\X_2^\top\bxi)^2],\nn
\end{align}
where the last three terms of the right hand side are convex and the first term obeys
\begin{align*}
\E[(\x_1^\top\W\X_1^\top\bxi)^2]&=\sigma^2\E[\x_1^\top\W\X_1^\top\X_1\W^\top \x_1]\\
&=\sigma^2\tr{\E[\x_1\x_1^\top\W\X_1^\top\X_1\W^\top]}\\
&=\sigma^2\tr{\bSb_{\x}\W\bSb_{\X}\W^\top}\\
&=\sigma^2\tf{\sqrt{\bSb_{\x}}\W\sqrt{\bSb_{\X}}}^2.
\end{align*}
Since noise level $\sigma>0$, using the full-rankness of covariance matrices $\bSb_{\x}$ and $\bSb_{\X}$, we conclude with strong convexity of $\E[(\x^\top\W\X^\top\bxi)^2]$. 

\noindent $\bullet$ \textbf{Claim 2:} Now recall that $\bSb_{\X}=\E[\X_1^\top\X_1]$ and set $\A=\X_1^\top\X_1-\bSb_{\X}$ and $\B=\X_2^\top\X_2+\bSb_{\X}$. Observe that $\E[\A]=0$. If the second claim holds, $\E[\X^\top\X]=\E[\A+\B]$. Note that $(\A,\bt_1,\x_1)$ are independent of each other and $(\B,\bt_2,\x_2)$. Using independence and $\E[\A]=0$, similarly write
\begin{align*}
\E[(\x^\top\W\X^\top\X\bt)^2]&=\E[(\x^\top\W\A\bt)^2]+\E[(\x^\top\W\B\bt)^2].
\end{align*}
Now using $\E[\bt_1]=\E[\x_1]=0$ and their independence from rest, these terms obeys 
\begin{align*}
&\E[(\x^\top\W\A\bt)^2]=\E[(\x_1^\top\W\A\bt_1)^2]+\E[(\x_1^\top\W\A\bt_2)^2]+\E[(\x_2^\top\W\A\bt_1)^2]+\E[(\x_2^\top\W\A\bt_2)^2]\\
&\E[(\x^\top\W\B\bt)^2]=\E[(\x_1^\top\W\B\bt_1)^2]+\E[(\x_1^\top\W\B\bt_2)^2]+\E[(\x_2^\top\W\B\bt_1)^2]+\E[(\x_2^\top\W\B\bt_2)^2].
\end{align*}
In both equations, the last three terms of the right hand side are convex. To proceed, we focus on the first terms. Using independence and setting $\bSi_{\X
}=\E[\X^\top\X]\succeq \bSb_{\X}\succ 0$, we note that
\begin{align*}
\E[(\x_1^\top\W\A\bt_1)^2]+\E[(\x_1^\top\W\B\bt_1)^2]=\E[(\x_1^\top\W\X^\top\X\bt_1)^2]
\end{align*}
where $\x_1,\bt_1,\X$ are independent and full-rank covariance. To proceed, note that 
\begin{align*}
\E[(\x_1^\top\W\X^\top\X\bt_1)^2]=\E[(\x_1^\top\W\bSi_{\X}\bt_1)^2]+\E[(\x_1^\top\W(\X^\top\X-\bSi_{\X})\bt_1)^2].
\end{align*}
Observing the convexity of the right hand side and focusing on the first term, we get
\begin{align*}
\E[(\x_1^\top\W\bSi_{\X}\bt_1)^2]=\tr{\bSb_{\x}\W\bSi_{\X}\bSb_{\bt}\bSi_{\X}\W^\top}=\tf{\sqrt{\bSb_x}\W\bSi_{\X}\sqrt{\bSb_\bt}}^2.
\end{align*}
Using the fact that covariance matrices, $\bSb_{\x},\bSi_{\X},\bSb_{\bt}$,  are full rank concludes the strong convexity proof of $\E[(\x^\top\W\X^\top\X\bt)^2]$.
\end{proof}
\section{Analysis of General Data Distribution}
In this section, we provide the proofs in Section~\ref{sec:main}, which focuses on solving Objective~\eqref{obj gd}. For the sake of clean notation, let $\Lc(\W):=\Lc_{\gd}(\Wc)$ and $g:=g_{\gd}$ in this section. 
\subsection{Supporting Results}
We begin by deriving the even moments of random variables.

\noindent $\bullet$ \textbf{$2n$'th moment of a normally distributed variable:} Let $u\sim\Nc(0,\sigma^2)$. Then we have
\begin{align}
\E[u^{2n}] = \sigma^{2n}(2n-1)!!.\label{moment 2n}
\end{align}
%

\noindent $\bullet$ \textbf{$4$'th moment: } Let $\ub\sim\Nc(0,\Iden_d)$. Then for any $\W,\W'\in\R^{d\times d}$, we have
\begin{align}
&\E\left[(\ub^\top\W\ub)(\ub^\top\W'\ub)\right]\nn\hspace*{300pt}\\
&=\E\left[\left(\sum_{i,j=1}^dW_{ij}u_iu_j\right)\left(\sum_{i,j=1}^dW'_{ij}u_iu_j\right)\right]\nn\\
&=\E\left[\left(\sum_{i=1}^dW_{ii}u_i^2\right)\left(\sum_{i=1}^dW'_{ii}u_i^2\right)\right]+\E\left[\left(\sum_{i\neq j}W_{ij}u_iu_j\right)\left(\sum_{i\neq j}W_{ij}'u_iu_j\right)\right]\nn\\
&=\sum_{i=1}^dW_{ii}W'_{ii}\E\left[u_i^4\right]+\sum_{i\neq j}W_{ii}W'_{jj}\E[u_i^2]\E[u_j^2]+\sum_{i\neq j}W_{ij}W'_{ij}\E[u_i^2]\E[u_j^2]+\sum_{i\neq j}W_{ij}W'_{ji}\E[u_i^2]\E[u_j^2]\nn\\
&=3\sum_{i=1}^dW_{ii}W'_{ii}+\sum_{i\neq j}W_{ii}W'_{jj}+\sum_{i\neq j}W_{ij}W'_{ij}+\sum_{i\neq j}W_{ij}W'_{ji}\nn\\
&=\sum_{i,j=1}^dW_{ii}W'_{jj}+\sum_{i,j=1}^dW_{ij}W'_{ij}+\sum_{i,j=1}^dW_{ij}W'_{ji}\nn\\
&=\trace{\W}\trace{\W'}+\trace{\W'\W^\top}+\trace{\W\W'}.\label{moment 4}
\end{align}

\noindent $\bullet$ \textbf{$4$'th cross-moment: } Let $\ub,\vb\sim\Nc(0,\Iden_d)$ and for any $\W\in\R^{d\times d}$, let $\La_{\W}=\W\odot\Iden_d$. Then we have
\begin{align}
&\E\left[(\ub^\top\W\vb\vb^\top\ub)^2\right]\nn\hspace*{300pt}\\
&=\E\left[\left(\sum_{i,j=1}^dW_{ij}u_iv_j\right)^2\left(\sum_{i=1}^du_iv_i\right)^2\right]\nn\\
&=\E\scalemath{0.9}{\left[\left(\sum_{i,j=1}^dW_{ij}^2u_i^2v_j^2+\sum_{i\neq i'}W_{ij}W_{i'j}u_iu_{i'}v_j^2+\sum_{j\neq j'}W_{ij}W_{ij'}u_i^2v_jv_{j'}+\sum_{i'\neq i,j'\neq j}W_{ij}W_{i'j'}u_iu_{i'}v_jv_{j'}\right)\left(\sum_{i=1}^du_i^2v_i^2+\sum_{i\neq j}u_iu_jv_iv_j\right)\right]}\nn\\
&=\E\left[\left(\sum_{i,j=1}^dW_{ij}^2u_i^2v_j^2\right)\left(\sum_{i=1}^du_i^2v_i^2\right)+\left(\sum_{i\neq j}W_{ij}W_{ji}u_i^2u_j^2v_i^2v_j^2\right)\right]\nn\\
&=\E\left[\left(\sum_{i=1}^dW_{ii}^2u_i^2v_i^2+\sum_{i\neq j}W_{ij}^2u_i^2v_j^2\right)\left(\sum_{i=1}^du_i^2v_i^2\right)\right]+\sum_{i\neq j}W_{ij}W_{ji}\nn\\
&=\E\scalemath{0.9}{\left[\left(\sum_{i=1}^dW_{ii}^2u_i^4v_i^4+\sum_{i\neq j}W_{ii}^2u_i^2v_i^2u_j^2v_j^2\right)\right]+\E\left[\left(\sum_{i\neq j}W_{ij}^2u_i^4v_j^2v_i^2+\sum_{i\neq j}W_{ij}^2u_i^2v_j^4u_j^2+\sum_{i\neq j\neq k}W_{ij}^2u_i^2v_j^2u_k^2v_k^2\right)\right]+\sum_{i\neq j}W_{ij}W_{ji}}\nn\\
&=9\sum_{i=1}^dW_{ii}^2+(d-1)\sum_{i=1}^dW_{ii}^2+6\sum_{i\neq j}W_{ij}^2+(d-2)\sum_{i \neq j}W_{ij}^2+\sum_{i\neq j}W_{ij}W_{ji}\nn\\
&=3\sum_{i=1}^dW_{ii}^2+(d+4)\sum_{i,j=1}^dW_{ij}^2+\sum_{i,j=1}^dW_{ij}W_{ji}\nn\\
&=3\tr{\La_{\W}^2}+(d+4)\tr{\W\W^\top}+\tr{\W^2}.\label{cross moment 4}
\end{align}

\noindent $\bullet$ \textbf{$6$'th moment: } Let $\ub\sim\Nc(0,\Iden_d)$. Then for any $\W,\W'\in\R^{d\times d}$, we have
\begin{align}
&\E\left[(\ub^\top\W\ub)(\ub^\top\W'\ub)\tn{\ub}^2\right]\nn\hspace*{300pt}\\
&=\E\left[\left(\sum_{i,j=1}^dW_{ij}u_iu_j\right)\left(\sum_{i,j=1}^dW'_{ij}u_iu_j\right)\left(\sum_{i=1}^du_i^2\right)\right]\nn\\
&=\E\left[\left(\sum_{i=1}^dW_{ii}u_i^2\right)\left(\sum_{i=1}^dW'_{ii}u_i^2\right)\left(\sum_{i=1}^du_i^2\right)\right]+\E\left[\left(\sum_{i\neq j}W_{ij}u_iu_j\right)\left(\sum_{i\neq j}W_{ij}'u_iu_j\right)\left(\sum_{i=1}^du_i^2\right)\right]\nn\\
&=\sum_{i=1}^dW_{ii}W'_{ii}\E\left[u_i^4\left(\sum_{i'=1}^du_{i'}^2\right)\right]+\sum_{i\neq j}W_{ii}W'_{jj}\E\left[u_i^2u_j^2\left(\sum_{i'=1}^du_{i'}^2\right)\right]\nn\\
&\quad+\sum_{i\neq j}W_{ij}W'_{ij}\E\left[u_i^2u_j^2\left(\sum_{i'=1}^du_{i'}^2\right)\right]+\sum_{i\neq j}W_{ij}W'_{ji}\E\left[u_i^2u_j^2\left(\sum_{i'=1}^du_{i'}^2\right)\right]\nn\\
&=(d+4)\left(3\sum_{i=1}^dW_{ii}W'_{ii}+\sum_{i\neq j}W_{ii}W'_{jj}+\sum_{i\neq j}W_{ij}W'_{ij}+\sum_{i\neq j}W_{ij}W'_{ji}\right)\label{moment 6 eq 1}\\
&=(d+4)\left(\sum_{i,j=1}^dW_{ii}W'_{jj}+\sum_{i,j=1}^dW_{ij}W'_{ij}+\sum_{i,j=1}^dW_{ij}W'_{ji}\right)\nn\\
&=(d+4)\left(\trace{\W}\trace{\W'}+\trace{\W'\W^\top}+\trace{\W\W'}\right),\label{moment 6}
\end{align}
where \eqref{moment 6 eq 1} is obtained by following
\begin{align*}
    &\E\left[u_i^4\left(\sum_{i'=1}^du_{i'}^2\right)\right]=\E[u^6]+(d-1)\E[u^4]\E[u^2]=3(d+4),\\
    &\E\left[u_i^2u_j^2\left(\sum_{i'=1}^du_{i'}^2\right)\right]=2\E[u^4]\E[u^2]+(d-2)\E[u^2]\E[u^2]\E[u^2]=d+4.
\end{align*}


\noindent $\bullet$ \textbf{$8$'th moment: } Let $\ub\sim\Nc(0,\Iden_d)$. Then for any $\W,\W'\in\R^{d\times d}$, we have
\begin{align}
&\E\left[(\ub^\top\W\ub)(\ub^\top\W'\ub)\tn{\ub}^4\right]\nn\hspace*{300pt}\\
&=\E\left[\left(\sum_{i,j=1}^dW_{ij}u_iu_j\right)\left(\sum_{i,j=1}^dW'_{ij}u_iu_j\right)\left(\sum_{i,j=1}^du_i^2u_j^2\right)\right]\nn\\
&=\E\scalemath{0.9}{\left[\left(\sum_{i=1}^dW_{ii}u_i^2\right)\left(\sum_{i=1}^dW'_{ii}u_i^2\right)\left(\sum_{i=1}^du_i^4+\sum_{i\neq j}u_i^2u_j^2\right)\right]+\E\left[\left(\sum_{i\neq j}W_{ij}u_iu_j\right)\left(\sum_{i\neq j}W_{ij}'u_iu_j\right)\left(\sum_{i=1}^du_i^4+\sum_{i\neq j}u_i^2u_j^2\right)\right]}\nn\\
&=\sum_{i=1}^dW_{ii}W'_{ii}\E\left[u_i^4\left(\sum_{i'=1}^du_{i'}^4+\sum_{i'\neq j'}u_{i'}^2u_{j'}^2\right)\right]+\sum_{i\neq j}W_{ii}W'_{jj}\E\left[u_i^2u_j^2\left(\sum_{i'=1}^du_{i'}^4+\sum_{i'\neq j'}u_{i'}^2u_{j'}^2\right)\right]\nn\\
&\quad+\sum_{i\neq j}W_{ij}W'_{ij}\E\left[u_i^2u_j^2\left(\sum_{i'=1}^du_{i'}^4+\sum_{i'\neq j'}u_{i'}^2u_{j'}^2\right)\right]+\sum_{i\neq j}W_{ij}W'_{ji}\E\left[u_i^2u_j^2\left(\sum_{i'=1}^du_{i'}^4+\sum_{i'\neq j'}u_{i'}^2u_{j'}^2\right)\right]\nn\\
&=(d+4)(d+6)\left(3\sum_{i=1}^dW_{ii}W'_{ii}+\sum_{i\neq j}W_{ii}W'_{jj}+\sum_{i\neq j}W_{ij}W'_{ij}+\sum_{i\neq j}W_{ij}W'_{ji}\right)\label{moment 8 eq 1}\\
&=(d+4)(d+6)\left(\sum_{i,j=1}^dW_{ii}W'_{jj}+\sum_{i,j=1}^dW_{ij}W'_{ij}+\sum_{i,j=1}^dW_{ij}W'_{ji}\right)\nn\\
&=(d+4)(d+6)\left(\trace{\W}\trace{\W'}+\trace{\W'\W^\top}+\trace{\W\W'}\right).\label{moment 8}
\end{align}
where \eqref{moment 8 eq 1} is obtained by following 
\begin{align*}
    &\E\left[u_i^4\left(\sum_{i'=1}^du_{i'}^4+\sum_{i'\neq j'}u_{i'}^2u_{j'}^2\right)\right]\\
    &=\E[u^8]+(d-1)\E[u^4]\E[u^4]+2(d-1)\E[u^6]\E[u^2]+(d-1)(d-2)\E[u^4]\E[u^2]\E[u^2]\\
    &=105+9(d-1)+30(d-1)+3(d-1)(d-2)\\
    &=3(d+4)(d+6),\\
    &\E\left[u_i^2u_j^2\left(\sum_{i'=1}^du_{i'}^4+\sum_{i'\neq j'}u_{i'}^2u_{j'}^2\right)\right]\\
    &=2\E[u^6]\E[u^2]+(d-2)\E[u^4](\E[u^2])^2+2\E[u^4]\E[u^4]+4(d-2)\E[u^4](\E[u^2])^2+(d-2)(d-3)(\E[u^2])^4\\
    &=30+3(d-2)+18+12(d-2)+(d-2)(d-3)\\
    &=(d+4)(d+6).
\end{align*}
\subsection{Independent Data with General Covariance}\label{app:independent}
\begin{proofof}{Theorem~\ref{thm:independent}}
Consider a general  independent linear model as defined in \eqref{data ind} where $\bSi_{\x}$ and $\bSi_{\bt}$ are full-rank feature and task convariance matrices and 
\[
\x\sim\Nc(0,\bSi_{\x}),\quad\bt\sim\Nc(0,\bSi_{\bt}),\quad\xi\sim\Nc(0,\sigma^2),\quad\text{and}\quad y=\x^\top\bt+\xi.
\]
Let 
\[
\X=[\x_1~\cdots~\x_n]^\top,\quad\bxi=[\xi_1~\cdots~\xi_n]^\top,\quad\text{and}\quad \y=[y_1~\cdots~y_n]^\top=\X\bt+\bxi.
\]

To simplify and without loss of generality, let $\xb=\bSi_{\x}^{-1/2}\x$, $\Xb=\X\bSi_{\x}^{-1/2}$,  $\btb=\bSi_{\x}^{1/2}\bt$ where we have 
\[
\xb\sim\Nc(0,\Iden),\qquad\btb\sim\Nc(0,\bSi_{\x}^{1/2}\bSi_{\bt}\bSi_{\x}^{1/2})
\]
and
\[
y=\xb^\top\btb+\xi,\qquad\y=\Xb\btb+\bxi.
\]

Then recap the loss from \eqref{obj gd}, and we obtain
\begin{align}
    \Lc(\W)&=\E\left[(y-g(\Z))^2\right]\nn\\
    &=\E\left[\left(\x^\top\bt+\xi-\x^\top\W\X^\top(\X\bt+\bxi)\right)^2\right]\nn\\
    &=\E\left[(\x^\top\bt-\x^\top\W\X^\top\X\bt)^2+2(\x^\top\bt-\x^\top\W\X^\top\X\bt)(\xi-\x^\top\W\X^\top\bxi)+(\xi-\x^\top\W\X^\top\bxi)^2\right]\nn\\
    &=\E\left[(\x^\top\bt-\x^\top\W\X^\top\X\bt)^2\right]+\E\left[(\x^\top\W\X^\top\bxi)^2\right]+\sigma^2,\label{loss expanded}
\end{align}
where the last equality comes from the independence of label noise $\xi,\bxi$.

We first consider the following term
\begin{align*}
    \E\left[(\x^\top\W\X^\top\bxi)^2\right]=\E\left[(\xb^\top(\bSi_{\x}^{1/2}\W\bSi_{\x}^{1/2})\Xb^\top\bxi)^2\right]=n\sigma^2\cdot\trace{\Wb\Wb^\top}
\end{align*}
where we define $\Wb=\bSi_{\x}^{1/2}\W\bSi_{\x}^{1/2}$. 
Next, focus on the following
\begin{align*}
    \E\left[(\x^\top\bt-\x^\top\W\X^\top\X\bt)^2\right]&=\E\left[(\xb^\top\btb-\xb^\top\Wb\Xb^\top\Xb\btb)^2\right]\\
    &=\E\left[\left(\xb^\top\left(\Iden-\Wb\Xb^\top\Xb\right)\btb\right)^2\right]\\
    &=\trace{\E\left[\left(\Iden-\Wb\Xb^\top\Xb\right)\bSi\left(\Iden-\Wb\Xb^\top\Xb\right)^\top\right]}\\
    &=\trace{\bSi}-\tr{\bSi(\Wb+\Wb^\top)\E[\Xb^\top\Xb]}+\tr{\Wb^\top\Wb\E[\Xb^\top\Xb\bSi\Xb^\top\Xb]}\\
    &=\trace{\bSi}-2n\cdot\tr{\bSi\Wb}+\tr{\Wb^\top\Wb\E[\Xb^\top\Xb\bSi\Xb^\top\Xb]},
\end{align*}
where $\bSi:=\bSi_{\x}^{1/2}\bSi_{\bt}\bSi_{\x}^{1/2}$.

Let $\xb_i\in\R^n$ be the $i$'th column of $\Xb$ and $\bSi_{ij}$ be the $(i,j)$'th entry of $\bSi$. Then the $(i,j)$ entry of matrix $\Xb^\top\Xb\bSi\Xb^\top\Xb$ is
\[
(\Xb^\top\Xb\bSi\Xb^\top\Xb)_{ij}=\sum_{k=1}^d\sum_{p=1}^d\bSi_{kp}\xb_i^\top\xb_k\xb_p^\top\xb_j.
\]
Then we get
\begin{align*}
    i\neq j:\quad\E\left[\left(\Xb^\top\Xb\bSi\Xb^\top\Xb\right)_{ij}\right]&=\bSi_{ij}\E[\xb_i^\top\xb_i\xb_j^\top\xb_j]+\bSi_{ji}\E[\xb_i^\top\xb_j\xb_i^\top\xb_j]=n^2\bSi_{ij}+n\bSi_{ji}\\
    i=j:\quad\E\left[\left(\Xb^\top\Xb\bSi\Xb^\top\Xb\right)_{ii}\right]&=\bSi_{ii}\E\left[\xb_i^\top\xb_i\xb_i^\top\xb_i\right]+\sum_{j\neq i}\bSi_{jj}\E\left[\xb_i^\top\xb_j\xb_j^\top\xb_i\right]\\
&=\bSi_{ii}\E\left[(x_{i1}^2+\cdots+x_{in}^2)^2\right]+n\sum_{j\neq i}\bSi_{jj}\\
&=\bSi_{ii}(3n+n(n-1))+n\sum_{j\neq i}\bSi_{jj}\\
&=n\left(\bSi_{ii}(n+1)+\sum_{j=1}^d\bSi_{jj}\right)\\
&=n\left(\bSi_{ii}(n+1)+\tr{\bSi}\right).
\end{align*}
Therefore 
\[
\E[\Xb^\top\Xb\bSi\Xb^\top\Xb]=n(n+1)\bSi+n\cdot\tr{\bSi}\Iden.
\]
Combining all together results in 
\begin{align}
\Lc(\W)&=\tr{\bSi}-2n\tr{\bSi\Wb}+n(n+1)\tr{\bSi\Wb^\top\Wb}+n(\tr{\bSi}+\sigma^2)\tr{\Wb\Wb^\top}+\sigma^2,\nn\\
&=M-2n\tr{\bSi\Wb}+n(n+1)\tr{\bSi\Wb^\top\Wb}+nM\tr{\Wb\Wb^\top},\label{loss}
\end{align}
where $M:=\tr{\bSi}+\sigma^2$. 
Setting $\nabla_{\Wb}\Lc(\W)=0$ returns
\[
-2n\cdot\bSi+2n(n+1)\cdot\bSi\Wb+2nM\Wb=0\Longrightarrow\Wb_\st=\left((n+1)\Iden+M\bSi^{-1}\right)^{-1}.
\]
Then we have 
\[
\W_\st=\bSi_{\x}^{-1/2}\left((n+1)\Iden+M\bSi^{-1}\right)^{-1}\bSi_{\x}^{-1/2}
\]
and 
$$\Lc_\st=\Lc(\W_\st)=M-n\tr{((n+1)\bSi^{-1}+M\bSi^{-2})^{-1}}.$$
\end{proofof}
\subsection{Retrieval Augmented Generation with $\alpha$ Correlation}
\label{app:rag}
In this section, we consider the retrieval augmented generation (RAG) linear model similar to \eqref{data rag}, where we first draw the query vector $\x$ and task vector $\bt$ via
\[
\x\sim\Nc(0,\Iden)\quad\text{and}\quad\bt\sim\Nc(0,\Iden).
\]
We then draw data $(\x_i)_{i=1}^n$ to be used in-context according to the rule $\text{corr\_coef}(\x,\x_i)\geq \alpha\geq0$. Hence, for $i\leq n$ we sample 
\begin{align}
\x_i\bgl\x\sim\Nn(\alpha\x,\gamma^2\Iden),\quad\xi_i\sim\Nc(0,\sigma^2)\quad\text{and}\quad y_i=\x_i^\top \bt+\xi_i,\label{data rag app}
\end{align}
which results in \eqref{data rag} by setting $\gamma^2=1-\alpha^2$.
\begin{theorem}[Extended version of Theorem~\ref{thm:rag}]\label{thm:rag exact}
    Consider linear model as defined in \eqref{data rag app}. Recap the objective from \eqref{obj gd} and let $\W_\st:=\arg\min_{\W}\Lc_{\gd}(\W)$, and $\Lc_\st=\Lc_{\gd}(\W_\st)$. 
    Then $\Ws$ and $\Lcs$ satisfy
    \begin{align}
    \W_\st=c\Iden\qquad\text{and}\qquad\Lc_\st=d+\sigma^2-cnd(\alpha^2(d+2)+\gamma^2)\label{formula rag app 1}
    \end{align}
    where 
    $$c=\frac{\alpha^2(d+2)+\gamma^2}{\alpha^4n(d+2)(d+4)+\alpha^2\gamma^2(d+2)(d+2n+3)+\gamma^4(d+n+1)+\sigma^2(\alpha^2(d+2)+\gamma^2)}.$$
Suppose $\alpha=\order{1/\sqrt{d}}$, $d/n=\order{1}$ and $d$ is sufficiently large. Let $\kappa=\alpha^2d+1$ and $\gamma^2=1-\alpha^2$. Then $\Ws$ and $\Lcs$ have approximate forms
    \begin{align}
    \W_\st\approx\frac{1}{\kappa n+d+\sigma^2}\Iden\qquad\text{and}\qquad
\Lc_\st\approx d+\sigma^2-\frac{\kappa nd}{\kappa n+d+\sigma^2}.\label{formula rag app 2}
\end{align}
\end{theorem}
\begin{proof}
Here, for clean notation and without loss of generality, we define and rewrite \eqref{data rag app} via
\[
\g_i\sim\Nc(0,\Iden),\quad\xi_i\sim\Nc(0,\sigma^2)\quad\text{and}\quad\x_i=\alpha\x+\gamma\g_i,\quad y_i=(\alpha\x+\gamma\g_i)^\top\bt+\xi_i.
\]
Then we obtain
\begin{align}
    \Lc(\W)&=\E\left[(y-g(\Z))^2\right]\nn\\
    &=\E\left[\left(\x^\top\bt+\xi-\x^\top\W\X^\top(\X\bt+\bxi)\right)^2\right]\nn\\
    &=\E\left[(\x^\top\bt-\x^\top\W\X^\top\X\bt)^2+2(\x^\top\bt-\x^\top\W\X^\top\X\bt)(\xi-\x^\top\W\X^\top\bxi)+(\xi-\x^\top\W\X^\top\bxi)^2\right]\nn\\
    &=\E\left[(\x^\top\bt-\x^\top\W\X^\top\X\bt)^2\right]+\E\left[(\x^\top\W\X^\top\bxi)^2\right]+\sigma^2\label{ind loss eq 1}.
\end{align}
To begin with, let
\begin{align}
N_1 = \tr{\W}^2+\tr{\W\W^\top}+\tr{\W^2},\quad N_2=\tr{\W\W^\top},\quad\text{and}\quad N_3=\tr{\W}.\nn
\end{align}
We first focus on the second term in \eqref{ind loss eq 1}
\begin{align*}
    \E\left[(\x^\top\W\X^\top\bxi)^2\right]&=\E\left[\left(\sum_{i=1}^n{\xi_i}\x^\top\W(\alpha\x+\gamma\g_i)\right)^2\right]\\
    &=n\sigma^2\E\left[\x^\top\W(\alpha\x+\gamma\g)(\alpha\x+\gamma\g)^\top\W^\top\x\right]\\
    &=n\sigma^2\left(\alpha^2\E[\x^\top\W\x\x^\top\W^\top\x]+\gamma^2\E[\x^\top\W\g\g^\top\W^\top\x]\right)\\
    &=n\sigma^2\left(\alpha^2N_1+\gamma^2N_2\right).\tag{It follows \eqref{moment 4} and independence of $\x,\g$.}
\end{align*}
Next, 
the first term in \eqref{ind loss eq 1} can be decomposed into 
\begin{align*}
    \E\left[(\x^\top\bt-\x^\top\W\X^\top\X\bt)^2\right]&=\underset{(a)}{\underbrace{\E\left[(\x^\top\bt)^2\right]}}+\underset{(b)}{\underbrace{\E\left[(\x^\top\W\X^\top\X\bt)^2\right]}}-2\underset{(c)}{\underbrace{\E\left[\x^\top\bt\x^\top\W\X^\top\X\bt\right]}}.
\end{align*}
In the following, we consider solving $(a)$-$(c)$ sequentially. 
\begin{align*}
    (a):\quad&\E\left[(\x^\top\bt)^2\right]=d.\hspace*{500pt}
\end{align*}
\begin{align*}
    (b):\quad&\E\left[(\x^\top\W\X^\top\X\bt)^2\right]\\
    &=\E\left[\left(\x^\top\W\sum_{i=1}^n(\alpha\x+\gamma\g_i)(\alpha\x+\gamma\g_i)^\top\bt\right)^2\right]\hspace*{500pt}\\
    &=\E\left[\left(\sum_{i=1}^n\x^\top\W(\alpha^2\x\x^\top+\gamma^2\g_i\g_i^\top+\alpha\gamma\x\g_i^\top+\alpha\gamma\g_i\x^\top)\bt\right)^2\right]\\
    &=\scalemath{0.85}{\alpha^4n^2\E\left[(\x^\top\W\x\x^\top\bt)^2\right]+\gamma^4\E\left[\left(\sum_{i=1}^n\x^\top\W\g_i\g_i^\top\bt\right)^2\right]+\alpha^2\gamma^2\E\left[\left(\sum_{i=1}^n\x^\top\W\x\g_i^\top\bt\right)^2\right]+\alpha^2\gamma^2\E\left[\left(\sum_{i=1}^n\x^\top\W\g_i\x^\top\bt\right)^2\right]}\\
    &\quad\scalemath{0.9}{+2\alpha^2\gamma^2n^2\E\left[\x^\top\W\x\x^\top\bt\bt^\top\g\g^\top\W^\top\x\right]+2\alpha^2\gamma^2 n\E\left[\x^\top\W\x\g^\top\bt\x^\top\W\g\x^\top\bt\right]}\\
    &=\scalemath{0.9}{\left(\alpha^4n^2(d+4)N_1+\gamma^4n(d+n+1)N_2\right)+\left(\alpha^2\gamma^2ndN_1+\alpha^2\gamma^2n(d+2)N_2\right)+\left(2\alpha^2\gamma^2n^2N_1+2\alpha^2\gamma^2nN_1\right)}\\
    &=\scalemath{0.9}{\left(\alpha^4n^2(d+4)+\alpha^2\gamma^2n(2n+d+2)\right)N_1+\left(\alpha^2\gamma^2n(d+2)+\gamma^4n(d+n+1)\right)N_2}\\
    &=A_1N_1+A_2N_2.
\end{align*}
\begin{align*}
    (c):\quad&\E\left[\x^\top\bt\x^\top\W\X^\top\X\bt\right]=\E\left[\sum_{i=1}^n\x^\top\bt\x^\top\W(\alpha\x+\gamma\g_i)(\alpha\x+\gamma\g_i)^\top\bt\right]\hspace*{300pt}\\
    &=\E\left[\sum_{i=1}^n\x^\top\bt\x^\top\W(\alpha^2\x\x^\top+\gamma^2\g_i\g_i^\top+\alpha\gamma\x\g_i^\top+\alpha\gamma\g_i\x^\top)\bt\right]\\
    &=\alpha^2n\E\left[\x^\top\bt\x^\top\W\x\x^\top\bt\right]+\gamma^2n\E\left[\x^\top\bt\x^\top\W\g\g^\top\bt\right]\\
    &=\alpha^2n(d+2)\tr{\W}+\gamma^2n\tr{\W}\\
    &=\left(\alpha^2n(d+2)+\gamma^2n\right)N_3\\
    &=A_3N_3.
\end{align*}
Here, $(b)$ utilizes the $4$'th and $6$'th moment results \eqref{moment 4} and \eqref{moment 6} and we define
\begin{align*}
&A_1=\alpha^4n^2(d+4)+\alpha^2\gamma^2n(2n+d+2)\\
&A_2=\alpha^2\gamma^2n(d+2)+\gamma^4n(d+n+1)\\
&A_3=\alpha^2n(d+2)+\gamma^2n.
\end{align*}
Then combining all together results in
\begin{align*}
\Lc(\W)=A_1N_1+A_2N_2-2A_3N_3+n\sigma^2(\alpha^2N_1+\gamma^2N_2)+d+\sigma^2.
\end{align*}
To find the optimal solution, set $\nabla\Lc(\W)=0$ and we obtain
\begin{align}
    A_1\nabla N_1+A_2\nabla N_2-2A_3\nabla N_3+n\sigma^2(\alpha^2\nabla N_1+\gamma^2\nabla N_2)=0.\label{aug diff 1}
\end{align}
Note that we have
\begin{align*}
    &\nabla N_1=\nabla\left(\tr{\W}^2+\tr{\W\W^\top}+\tr{\W^2}\right)=2\tr{\W}\Iden+2\W+2\W^\top\\
    &\nabla N_2=\nabla\tr{\W\W^\top}=2\W\\
    &\nabla N_3=\nabla\tr{\W}=\Iden.
\end{align*}
Therefore, \eqref{aug diff 1} returns
\begin{align}
    2A_1\left(\tr{\W}\Iden+\W+\W^\top\right)+2A_2\W-2A_3+2n\sigma^2(\alpha^2(\tr{\W}\Iden+\W+\W^\top)+\gamma^2\W)\Iden=0,\label{aug diff 2}
\end{align}
which implies that the optimal solution $\W_\st$ has the form of $c\Iden$ for some constant $c$. Then suppose $\W_\st=c\Iden$, we have $\tr{\W}=cd$ and \eqref{aug diff 2} returns
\[
2A_1(d+2)c\Iden+2A_2c\Iden-2A_3\Iden+2n\sigma^2(\alpha^2(d+2)c\Iden+\gamma^2c\Iden)=0
\]
\begin{align*}
\Longrightarrow c&=\frac{A_3}{A_1(d+2)+A_2+n\sigma^2(\alpha^2(d+2)+\gamma^2)}\\
&=\frac{\alpha^2(d+2)+\gamma^2}{\alpha^4n(d+2)(d+4)+\alpha^2\gamma^2(d+2)(d+2n+3)+\gamma^4(d+n+1)+\sigma^2(\alpha^2(d+2)+\gamma^2)}.
\end{align*}
Then the optimal loss is obtained by setting $\W_\st=c\Iden$ and 
\begin{align*}
\Lc_\st=\Lc(\W_\st)&=A_1c^2d(d+2)+A_2c^2d-2A_3cd+n\sigma^2c^2d(\alpha^2(d+2)+\gamma^2)+d+\sigma^2\\
&=c^2d\left(A_1(d+2)+A_2+n\sigma^2(\alpha^2(d+2)+\gamma^2)\right)-2A_3cd+d+\sigma^2\\
&=d+\sigma^2-A_3cd.
\end{align*}
It completes the proof of \eqref{formula rag app 1}. 
Now if assuming $\alpha=\order{1/\sqrt{d}}$, $d/n=\order{1}$ and sufficiently large dimension $d$, we have the approximate
\begin{align*}
c&\approx\frac{\alpha^2d+1}{\alpha^4d^2n+\alpha^2d(d+2n)+(d+n)+\sigma^2(\alpha^2d+1)}\\
&=\frac{\alpha^2d+1}{(\alpha^2d+1)^2n+(\alpha^2d+1)d+\sigma^2(\alpha^2d+1)}\\
&=\frac{1}{(\alpha^2d+1)n+d+\sigma^2}
\end{align*}
and 
\begin{align*}
    \Lc_\st&\approx d+\sigma^2-\frac{(\alpha^2d+1)nd}{(\alpha^2d+1)n+d+\sigma^2}.
\end{align*}
\end{proof}

\subsection{Task-feature Alignment with $\alpha$ Correlation}\label{app:task feature}

In this section, we consider the task-feature alignment data model similar to \eqref{data tfa}, where we first draw task vector $\bt$ via
\[
\bt\sim\Nc(0,\Iden).
\]
Then we generate examples $(\x_i,y_i)_{i=1}^{n+1}$ according to the rule $\text{corr\_coef}(\x_i,\bt)\geq \alpha\geq0$ via
\begin{align}
\x_i\bgl\bt\sim\Nn(\alpha\bt,\Iden),\quad\xi_i\sim\Nc(0,\sigma^2)\quad\text{and}\quad y_i=\gamma\cdot\x_i^\top \bt+\xi_i,\label{data feature app}
\end{align}
which results in \eqref{data tfa} by setting $\gamma^2=1/(\alpha^2d+1)$.
\begin{theorem}[Extended version of Theorem~\ref{thm:feature}]\label{thm:feature app}
    Consider linear model as defined in \eqref{data feature app}. Recap the objective from \eqref{obj gd} and let $\W_\st:=\arg\min_{\W}\Lc_{\gd}(\W)$, and $\Lc_\st=\Lc_{\gd}(\W_\st)$. 
    Then $\W_\st$ and $\Lc_\st$ satisfy
    \begin{align}
        \W_\st=c\Iden\qquad\text{and}\qquad\Lc_\st=d\gamma^2(\Delta_0\alpha^2+1)+\sigma^2-cnd\gamma^2(\Delta_1\alpha^4+2\Delta_0\alpha^2+1)\label{formula task feature app 1}
    \end{align}
    where
    \[
    c=\frac{\Delta_1\alpha^4+2\Delta_0\alpha^2+1}{\Delta_2\alpha^6+\Delta_3\alpha^4+\Delta_4\alpha^2+(d+n+1)+\sigma^2(\Delta_0\alpha^4+2\alpha^2+1)/\gamma^2}
    \]
    and 
    \[\begin{cases}
    \Delta_0=d+2\\
    \Delta_1=(d+2)(d+4)\\
    \Delta_2=(d+2)(d+4)(d+6)n\\
    \Delta_3=(d+2)(d+4)(3n+4)\\
    \Delta_4=(d+2)(3n+d+3)+(d+8).  
    \end{cases}
    \]
    Suppose $\alpha=\order{1/\sqrt{d}}$, $d/n=\order{1}$ and $d$ is sufficiently large. Let $\kappa=\alpha^2d+1$ and $\gamma^2=1/\kappa$. Then $\Ws$ and $\Lcs$ have approximate forms
    \begin{align}
    \W_\st\approx\frac{1}{\kappa n+(d+\sigma^2)/\kappa}\qquad\text{and}\qquad
    \Lc_\st\approx d+\sigma^2-\frac{\kappa nd}{\kappa n+(d+\sigma^2)/\kappa}.
    \label{formula task feature app 2}
    \end{align}
\end{theorem}
\begin{proof}
Here, for clean notation and without loss of generality, we define and rewrite \eqref{data feature app} via
\[
\g_i\sim\Nc(0,\Iden),\quad\xi_i\sim\Nc(0,\sigma^2)\quad\text{and}\quad\x_i=\alpha\bt+\g_i,\quad y_i=\gamma\x_i^\top\bt+\xi_i=\gamma\cdot(\alpha\bt+\g_i)^\top\bt+\xi_i.
\]
Recap the loss function from \eqref{obj gd}, we obtain
\begin{align}
    \Lc(\W)&=\E\left[(y-g(\Z))^2\right]\nn\\
    &=\E\left[\left(\gamma\x^\top\bt+\xi-\x^\top\W\X^\top(\gamma\X\bt+\bxi)\right)^2\right]\nn\\
    &=\E\left[\gamma^2(\x^\top\bt-\x^\top\W\X^\top\X\bt)^2+2\gamma(\x^\top\bt-\x^\top\W\X^\top\X\bt)(\xi-\x^\top\W\X^\top\bxi)+(\xi-\x^\top\W\X^\top\bxi)^2\right]\nn\\
    &=\gamma^2\E\left[(\x^\top\bt-\x^\top\W\X^\top\X\bt)^2\right]+\E\left[(\x^\top\W\X^\top\bxi)^2\right]+\sigma^2.\label{feature loss eq1}
\end{align}
Similar to Appendix~\ref{app:rag}, to begin with, let
\begin{align}
N_1 = \tr{\W}^2+\tr{\W\W^\top}+\tr{\W^2},\quad N_2=\tr{\W\W^\top},\quad\text{and}\quad N_3=\tr{\W},\nn
\end{align}
and additionally, given $\La_{\W}=\W\odot\Iden$, let
\begin{align}
N_4 = 3\tr{\La_{\W}^2}+(d+4)\tr{\W\W^\top}+\tr{\W^2}.\nn
\end{align}
We first focus on the second term in \eqref{feature loss eq1}
\begin{align*}
    \E\left[(\x^\top\W\X^\top\bxi)^2\right]&=\E\left[\left((\alpha\bt+\g)^\top\W\sum_{i=1}^n\xi_i(\alpha\bt+\g_i)\right)^2\right]\\
    &=n\sigma^2\E\left[\left((\alpha\bt+\g)^\top\W(\alpha\bt+\g')\right)^2\right]\\
    &=n\sigma^2\left(\alpha^4\E\left[(\bt^\top\W\bt)^2\right]+2\alpha^2\E\left[(\bt^\top\W\g')^2\right]+\E\left[(\g^\top\W\g')^2\right]\right)\\
    &=n\sigma^2\left(\alpha^4\left(\trace{\W}^2+\trace{\W^2}+\trace{\W\W^\top}\right)+(2\alpha^2+1)\trace{\W\W^\top}\right)\\
    &=n\sigma^2\left(\alpha^4N_1+(2\alpha^2+1)N_2\right).\tag{It follows \eqref{moment 4} and independence of $\bt,\g,\g'$.}
\end{align*}

Next, the first term of \eqref{feature loss eq1} (omitting $\gamma^2$) returns the following decomposition:
\begin{align*}
    \E\left[(\x^\top\bt-\x^\top\W\X^\top\X\bt)^2\right]&=\E\left[((\alpha\bt+\g)^\top(\bt-\W\X^\top\X\bt))^2\right]\\
    &=\E\left[\left(\alpha\bt^\top\bt-\alpha\bt^\top\W\X^\top\X\bt+\g^\top\bt-\g^\top\W\X^\top\X\bt\right)^2\right]\\
    &=\alpha^2\E[(\bt^\top\bt)^2]+\alpha^2\E[(\bt^\top\W\X^\top\X\bt)^2]+\E[(\g^\top\bt)^2]+\E[(\g^\top\W\X^\top\X\bt)^2]\\
    &\quad-2\alpha^2\E[\bt^\top\bt\bt^\top\W\X^\top\X\bt]-2\E[\bt^\top\g\g^\top\W\X^\top\X\bt]\\
    &=\alpha^2d(d+2)+\alpha^2\underset{(a)}{\underbrace{\E[(\bt^\top\W\X^\top\X\bt)^2]}}+d+\underset{(b)}{\underbrace{\E[(\g^\top\W\X^\top\X\bt)^2]}}\\
    &\quad-2\alpha^2\underset{(c)}{\underbrace{\E[\bt^\top\bt\bt^\top\W\X^\top\X\bt]}}-2\underset{(d)}{\underbrace{\E[\bt^\top\g\g^\top\W\X^\top\X\bt]}}.
\end{align*}
Consider solving $(a)$-$(d)$ sequentially as follows: 

To begin with, we use the following decomposition for all $(a)$-$(d)$:
\begin{align*}
    \X^\top\X\bt&=\sum_{i=1}^n\x_i\x_i^\top\bt\\
    &=\sum_{i=1}^n(\alpha\bt+\g_i)(\alpha\bt+\g_i)^\top\bt\\
    &=\sum_{i=1}^n\alpha^2\bt\bt^\top\bt+\alpha\bt\g_i^\top\bt+\alpha\g_i\bt^\top\bt+\g_i\g_i^\top\bt.
\end{align*}
Then, we have
\begin{align}
    (a):\quad&\E[(\bt^\top\W\X^\top\X\bt)^2]\nn\\
    &=\scalemath{0.85}{\E\left[\left(\sum_{i=1}^n\alpha^2\bt^\top\W\bt\bt^\top\bt+\alpha\bt^\top\W\bt\g_i^\top\bt+\alpha\bt^\top\W\g_i\bt^\top\bt+\bt^\top\W\g_i\g_i^\top\bt\right)^2\right]}\nn\\
    &=\scalemath{0.85}{\alpha^4n^2\E\left[\left(\bt^\top\W\bt\bt^\top\bt\right)^2\right]+\alpha^2\E\left[\left(\sum_{i=1}^n\bt^\top\W\bt\g_i^\top\bt\right)^2\right]+\alpha^2\E\left[\left(\sum_{i=1}^n\bt^\top\W\g_i\bt^\top\bt\right)^2\right]+\E\left[\left(\sum_{i=1}^n\bt^\top\W\g_i\g_i^\top\bt\right)^2\right]}\nn\\
    &\scalemath{0.85}{\quad+2\alpha^2n\E\left[\sum_{i=1}^n\bt^\top\W\bt\bt^\top\bt\bt^\top\W\g_i\g_i^\top\bt\right]+2\alpha^2\E\left[\sum_{i=1}^n\bt^\top\W\bt\g_i^\top\bt\bt^\top\W\g_i\bt^\top\bt\right]}\nn\\
    &=\scalemath{0.85}{\alpha^4n^2\E\left[\left(\bt^\top\W\bt\bt^\top\bt\right)^2\right]+\alpha^2n\E\left[\left(\bt^\top\W\bt{\g'}^\top\bt\right)^2\right]+\alpha^2n\E\left[\left(\bt^\top\W{\g'}\bt^\top\bt\right)^2\right]+\E\left[\left(\sum_{i=1}^n\bt^\top\W\g_i\g_i^\top\bt\right)^2\right]}\nn\\
    &\scalemath{0.85}{\quad+2\alpha^2n^2\E\left[\bt^\top\W\bt\bt^\top\bt\bt^\top\W\g'{\g'}^\top\bt\right]+2\alpha^2n\E\left[\bt^\top\W\bt\g_i^\top\bt\bt^\top\W\g_i\bt^\top\bt\right]}\nn\\
    &=\alpha^4n^2(d+4)(d+6)N_1+\alpha^2n(d+4)N_1+\alpha^2n(d+2)(d+4)N_2\label{feature eq 1}\\
    &\quad+n(n-1)N_1+nN_4\label{feature eq 2}\\
    &\quad+2\alpha^2n^2(d+4)N_1+2\alpha^2n(d+4)N_1\label{feature eq 3}\\
    &=\left(\alpha^2n(d+4)(\alpha^2n(d+6)+2n+3)+n(n-1)\right)N_1+\alpha^2n(d+2)(d+4)N_2+nN_4\\
    &=B_1N_1+B_2N_2+nN_4,\nn
\end{align}
where \eqref{feature eq 1} and \eqref{feature eq 3} utilize \eqref{moment 6} and \eqref{moment 8}, and \eqref{feature eq 2} is obtained via
\begin{align*}
    \E\left[\left(\sum_{i=1}^n\bt^\top\W\g_i\g_i^\top\bt\right)^2\right]&=n\E\left[\left(\bt^\top\W\g'{\g'}^\top\bt\right)^2\right]+n(n-1)\E\left[\bt^\top\W\g'{\g'}^\top\bt\bt^\top\W\g''{\g''}^\top\bt\right]\\
    &=nN_4+n(n-1)N_1,
\end{align*}
which follows \eqref{moment 4} and \eqref{cross moment 4}.

\begin{align}
(b):\quad&\E\left[(\g^\top\W\X^\top\X\bt)^2\right]\nn\\
&=\scalemath{0.85}{\E\left[\left(\sum_{i=1}^n\alpha^2\g^\top\W\bt\bt^\top\bt+\alpha\g^\top\W\bt\g_i^\top\bt+\alpha\g^\top\W\g_i\bt^\top\bt+\g^\top\W\g_i\g_i^\top\bt\right)^2\right]}\nn\\        
    &=\scalemath{0.85}{\alpha^4n^2\E\left[\left(\g^\top\W\bt\bt^\top\bt\right)^2\right]+\alpha^2\E\left[\left(\sum_{i=1}^n\g^\top\W\bt\g_i^\top\bt\right)^2\right]+\alpha^2\E\left[\left(\sum_{i=1}^n\g^\top\W\g_i\bt^\top\bt\right)^2\right]+\E\left[\left(\sum_{i=1}^n\g^\top\W\g_i\g_i^\top\bt\right)^2\right]}\nn\\
    &\scalemath{0.85}{\quad+2\alpha^2n\E\left[\sum_{i=1}^n\g^\top\W\bt\bt^\top\bt\g^\top\W\g_i\g_i^\top\bt\right]+2\alpha^2\E\left[\sum_{i=1}^n\g^\top\W\bt\g_i^\top\bt\g^\top\W\g_i\bt^\top\bt\right]}\nn\\
    &=\scalemath{0.85}{\alpha^4n^2\E\left[\left(\g^\top\W\bt\bt^\top\bt\right)^2\right]+\alpha^2n\E\left[\left(\g^\top\W\bt{\g'}^\top\bt\right)^2\right]+\alpha^2n\E\left[\left(\g^\top\W{\g'}\bt^\top\bt\right)^2\right]+\E\left[\left(\sum_{i=1}^n\g^\top\W\g_i\g_i^\top\bt\right)^2\right]}\nn\\
    &\quad\scalemath{0.85}{+2\alpha^2n^2\E\left[\g^\top\W\bt\bt^\top\bt\g^\top\W\g'{\g'}^\top\bt\right]+2\alpha^2n\E\left[\g^\top\W\bt\g_i^\top\bt\g^\top\W\g_i\bt^\top\bt\right]}\nn\\
    &=\alpha^4n^2(d+2)(d+4)N_2+\alpha^2n(d+2)N_2+\alpha^2nd(d+2)N_2+n(d+n+1)N_2\label{feature eq 4}\\
    &\quad+2\alpha^2n^2(d+2)N_2+2\alpha^2n(d+2)N_2\label{feature eq 5}\\
    &=\left(\alpha^2n(d+2)(\alpha^2n(d+4)+2n+d+3)+n(d+n-1)\right)N_2\nn\\
    &=B_3N_2,\nn
\end{align}
where \eqref{feature eq 4} and \eqref{feature eq 5} are obtained using \eqref{moment 4}, \eqref{moment 6} and
\begin{align*}
    \E\left[\left(\sum_{i=1}^n\g^\top\W\g_i\g_i^\top\bt\right)^2\right]&=n\E\left[\left(\g^\top\W\g'{\g'}^\top\bt\right)^2\right]+n(n-1)\E\left[\g^\top\W\g'{\g'}^\top\bt\g^\top\W\g''{\g''}^\top\bt\right]\\
    &=n(d+2)N_2+n(n-1)N_2=n(n+d+1)N_2.
\end{align*}
\begin{align*}
    (c):\quad&\E\left[\bt^\top\bt\bt^\top\W\X^\top\X\bt\right]\\
    &=n\E\left[\bt^\top\bt\bt^\top\W(\alpha\bt+\g')(\alpha\bt+\g')^\top\bt\right]\hspace*{300pt}\\
    &=\alpha^2n\E\left[\bt^\top\bt\bt^\top\W\bt\bt^\top\bt\right]+n\E\left[\bt^\top\bt\bt^\top\W{\g'}{\g'}^\top\bt\right]\\
    &=\alpha^2n(d+2)(d+4)\tr{\W}+n(d+2)\tr{\W}\\
    &=\left(\alpha^2n(d+2)(d+4)+n(d+2)\right)N_3\\
    &=B_4N_3.
\end{align*}
\begin{align*}(d):\quad&\E\left[\bt^\top\g\g^\top\W\X^\top\X\bt\right]\\
&=n\E\left[\bt^\top\g\g^\top\W(\alpha\bt+\g')(\alpha\bt+\g')^\top\bt\right]\hspace*{300pt}\\
    &=\alpha^2n\E\left[\bt^\top\g\g^\top\W\bt\bt^\top\bt\right]+n\E\left[\bt^\top\g\g^\top\W\g'{\g'}^\top\bt\right]\\
    &=\alpha^2n(d+2)\tr{\W}+n\tr{\W}\\
    &=\left(\alpha^2n(d+2)+n\right)N_3\\
    &=B_5N_3.
\end{align*}
Here we define 
\begin{align*}
    B_1&=\alpha^2n(d+4)(\alpha^2n(d+6)+2n+3)+n(n-1)\\
    B_2&=\alpha^2n(d+2)(d+4)\\
    B_3&=\alpha^2n(d+2)(\alpha^2n(d+4)+2n+d+3)+n(d+n-1)\\
    B_4&=\alpha^2n(d+2)(d+4)+n(d+2)\\
    B_5&=\alpha^2n(d+2)+n.
\end{align*}
Then combining all together results in 
\begin{align*}
\Lc(\W)&=\scalemath{0.85}{\gamma^2\left(\alpha^2d(d+2)+d+\alpha^2(B_1N_1+B_2N_2+nN_4)+B_3N_2-2\alpha^2 B_4N_3-2 B_5N_3\right)+n\sigma^2(\alpha^4N_1+(2\alpha^2+1)N_2)+\sigma^2}\\
&=\scalemath{0.85}{\gamma^2\left(\alpha^2B_1N_1+(\alpha^2B_2+B_3)N_2-2(\alpha^2B_4+B_5)N_3 +\alpha^2nN_4\right)+n\sigma^2(\alpha^4N_1+(2\alpha^2+1)N_2)+\gamma^2d\left(\alpha^2(d+2)+1\right)+\sigma^2}
\end{align*}
and differentiating it results in
\[
\nabla\Lc(\W)=\gamma^2\left(\alpha^2B_1\nabla N_1+(\alpha^2B_2+B_3)\nabla N_2-2(\alpha^2B_4+B_5)\nabla N_3 +\alpha^2n\nabla N_4\right)+n\sigma^2(\alpha^4\nabla N_1+(2\alpha^2+1)\nabla N_2).
\]
Similar to the proof in Appendix~\ref{app:rag}, $\W_\st$ has the form of $\W_\st=c\Iden$ and we have
\begin{align*}
    \nabla N_1&=\nabla\left(\tr{\W}^2+\tr{\W\W^\top}+\tr{\W^2}\right)=2\tr{\W}\Iden+2\W+2\W^\top=2c(d+2)\Iden\\
    \nabla N_2&=\nabla\tr{\W\W^\top}=2\W=2c\Iden\\
    \nabla N_3&=\nabla\tr{\W}=\Iden\\
    \nabla N_4&=\nabla\left(3\tr{\La_{\W}^2}+(d+4)\tr{\W\W^\top}+\tr{\W^2}\right)\\
    &=6\cdot\diag{\La_{\W}}+2(d+4)\W+2\W^\top\\
    &=2c(d+8)\Iden.
\end{align*}
Therefore, setting $\nabla\Lc(\W)=0$ returns
\[
\gamma^2\left(2c(d+2)\alpha^2B_1+2c(\alpha^2B_2+B_3)-2(\alpha^2B_4+B_5) +2c(d+8)\alpha^2n\right)+2cn\sigma^2(\alpha^4(d+2)+2\alpha^2+1)=0
\]
\begin{align*}
\Longrightarrow c&=\frac{\alpha^2 B_4+B_5}{(d+2)\alpha^2B_1+(\alpha^2B_2+B_3) +(d+8)\alpha^2n+n\sigma^2(\alpha^4(d+2)+2\alpha^2+1)/\gamma^2}\\
&=\scalemath{0.7}{\frac{\alpha^4n(d+2)(d+4)+2\alpha^2n(d+2)+n}{\alpha^6n^2(d+2)(d+4)(d+6)+\alpha^4n(d+2)(d+4)(3n+4)+\alpha^2n((d+2)(3n+d+3)+(d+8))+n(d+n+1)+n\sigma^2(\alpha^4(d+2)+2\alpha^2+1)/\gamma^2}}\\
&=\scalemath{0.7}{\frac{\alpha^4(d+2)(d+4)+2\alpha^2(d+2)+1}{\alpha^6n(d+2)(d+4)(d+6)+\alpha^4(d+2)(d+4)(3n+4)+\alpha^2((d+2)(3n+d+3)+(d+8))+(d+n+1)+\sigma^2(\alpha^4(d+2)+2\alpha^2+1)/\gamma^2}}.
\end{align*}
Then the optimal loss is obtained by setting $\W_\st=c\Iden$ and
\[
\Lc_\st=\Lc(\W_\st)=\gamma^2d(\alpha^2(d+2)+1)+\sigma^2-\gamma^2(\alpha^2B_4+B_5) cd.
\]
It completes the proof of \eqref{formula task feature app 1}. Now if assuming $\alpha=\order{1/\sqrt{d}}, d/n=\order{1}$, $\gamma^2=1/(\alpha^2d+1)$ and sufficiently large dimension $d$, we have the approximate
\begin{align*}
    c&\approx\frac{\alpha^4d^2+2\alpha^2d+1}{n\alpha^6d^3+3n\alpha^4d^2+(3n+d)\alpha^2d+d+n+\sigma^2(\alpha^4d+2\alpha^2+1)/\gamma^2}\\
    &\approx\frac{(\alpha^2d+1)^2}{n(\alpha^2d+1)^3+d(\alpha^2d+1)+\sigma^2(\alpha^2d+1)}\\
    &\approx\frac{1}{(\alpha^2d+1)n+(d+\sigma^2)/(\alpha^2d+1)}
\end{align*}
and
\begin{align*}
\Lc_\st&\approx \gamma^2d(\alpha^2d+1)+\sigma^2-\frac{\gamma^2(\alpha^2d+1)^2nd}{(\alpha^2d+1)n+(d+\sigma^2)/(\alpha^2d+1)}\\
&=d+\sigma^2-\frac{(\alpha^2d+1)nd}{(\alpha^2d+1)n+(d+\sigma^2)/(\alpha^2d+1)}.
\end{align*}
\end{proof}
\section{Analysis of Low-Rank Parameterization}
\subsection{Proof of Lemma~\ref{low rank attn}}
\label{app:low-rank}

\begin{proof}
Recall the loss function from \eqref{loss}
\begin{align*}
\Lc(\W)=M-2n\tr{\bSi\Wb}+n(n+1)\tr{\bSi\Wb^\top\Wb}+nM\tr{\Wb\Wb^\top}
\end{align*}
where $\Wb=\bSi_{\x}^{1/2}\W\bSi_{\x}^{1/2}$,  $\bSi=\bSi_{\x}^{1/2}\bSi_{\bt}\bSi_{\x}^{1/2}$ and $M=\tr{\bSi}+\sigma^2$. 
For any $\Wb$, let us parameterize $\Wb=\Ub\Eb\Ub^\top$ where $\Ub\in\R^{d\times r}$ denotes the eigenvectors of $\Wb$ and $\Eb\in\R^{r\times r}$ is a symmetric square matrix. We will first treat $\Ub$ as fixed and optimize $\Eb$. We will then optimize $\Ub$. Fixing $\Ub$, setting  $\bSb=\Ub^\top\bSi\Ub$, 
we obtain
\begin{align*}
\Lc(\Eb)&=M-2n\tr{\bSb\Eb}+n(n+1)\tr{\bSb\Eb^2}+nM\tr{\Eb^2}.
\end{align*}
Differentiating, we obtain 
\[
0.5n^{-1}\nabla\Lc(\Eb)=-\bSb+(n+1)\bSb\Eb+M\Eb.
\]
Setting $\nabla\Lc(\Eb)=0$ returns
\begin{align}
\Eb_\st=(M\Iden+(n+1)\bSb)^{-1}\bSb.\label{ebst eq}
\end{align}
Let $\lab_i$ denote the $i$'th largest eigenvalue of $\bSb$. Plugging in this value, we obtain the optimal risk as a function of $\Ub$ is given by
\begin{align}
\Lc_\st(\Ub)&=M-n\cdot\tr{\bSb\Eb_\st}=M-n\cdot\tr{(M\Iden+(n+1)\bSb)^{-1}\bSb^2}\\
&=M-n\sum_{i=1}^r \frac{\lab_i^2}{(n+1)\lab_i+M}=M-n\sum_{i=1}^r \frac{\lab_i}{n+1+M\lab_i^{-1}}.
\end{align}
Now observe that, the right hand side is strictly decreasing function of the eigenvalues $\lab_i$ of $\bSb=\Ub^\top\bSi\Ub$. Thus, to minimize $\Lc_\st(\Ub)$, we need to maximize $\sum_{i=1}^r \frac{\lab_i}{n+1+M\lab_i^{-1}}$. It follows from Cauchy interlacing theorem that $\lab_j\leq \la_i$ where $\la_i$ is the $i$'th largest eigenvalue of $\bSi$ since $\bSb$ is an orthogonal projection of $\bSi$ on $\Ub$. Consequently, we find the desired bound where
\[
\Lc_\st=M-n\sum_{i=1}^r \frac{\la_{i}}{n+1+M\la_{i}^{-1}}.
\]
The equality holds by setting $\Ub$ to be the top-$r$ eigenvectors of $\bSi$ and $\Eb=\Eb_\st(\Ub)$ to be the diagonal matrix according to \eqref{ebst eq}.
\end{proof}

\end{document}